\documentclass[twoside,11pt]{article}

\usepackage{jmlr2e}

\usepackage{booktabs}       % professional-quality tables
\usepackage{amsmath}
\usepackage{amssymb}
\usepackage{mathtools}
\usepackage[caption=false,font=footnotesize,labelfont=sf,textfont=sf]{subfig}
\usepackage{wrapfig}
\usepackage{caption}
\usepackage{multirow}
\usepackage{comment}
\usepackage{bm}

\usepackage{tikz}
\usetikzlibrary{arrows.meta, calc, positioning}
\newcommand{\circled}[1]{\raisebox{.5pt}{\textcircled{\raisebox{-1pt} {#1}}}}
\newcommand{\Boxed}[1]{\tikz[baseline=(char.base)]{\node[draw,inner sep=1pt] (char) {\scriptsize $#1$};}}

\newcommand{\R}{\mathbb{R}}
\renewcommand{\P}{\mathrm{P}}
\newcommand{\proj}{\mathcal{P}}
\newcommand{\p}{\bm{\pi}}

\usepackage[linesnumbered,ruled,vlined]{algorithm2e}

\SetKwComment{Comment}{$\triangleright$ }{}
\usepackage{lastpage}
\jmlrheading{XX}{2025}{1-\pageref{LastPage}}{XX; Revised XXX}{XXX}{21-0000}{Wenlong Lyu and Yuheng Jia and Hui Liu and Junhui Hou}

\ShortHeadings{Graph-based Clustering Revisited}{Lyu and Jia and Liu and Hou}

\begin{document}

\title{Graph-based Clustering Revisited: \\ A Relaxation of Kernel $k$-Means Perspective}

\author{\name Wenlong Lyu \email l\_w\_l@seu.edu.cn \\
       \addr School of Computer Science and Engineering \\
       Southeast University\\
       Nanjing, Jiangsu 211189, China
       \AND
       \name Yuheng Jia \email yhjia@seu.edu.cn \\
       \addr Division of Computer Science and Engineering\\
       Southeast University\\
       Nanjing, Jiangsu 211189, China
       \AND
       \name Hui Liu \email h2liu@sfu.edu.hk \\
       \addr Yam Pak Charitable Foundation School of Computing and Information Sciences\\
       Saint Francis University\\
       Hong Kong SAR 
       \AND
       \name Junhui Hou \email jh.hou@cityu.edu.hk \\
       \addr Department of Computer Science\\
       City University of Hong Kong\\
       Hong Kong SAR 
}

\editor{My editor}

\maketitle

\begin{abstract}%   <- trailing '%' for backward compatibility of .sty file
    The well-known graph-based clustering methods, including spectral clustering, symmetric non-negative matrix factorization, and doubly stochastic normalization, can be viewed as relaxations of the kernel $k$-means approach. 
    However, we posit that these methods excessively relax their inherent low-rank, nonnegative, doubly stochastic, and orthonormal constraints to ensure numerical feasibility, potentially limiting their clustering efficacy. 
    In this paper, guided by our theoretical analyses, we propose \textbf{Lo}w-\textbf{R}ank \textbf{D}oubly stochastic clustering (\textbf{LoRD}), a model that only relaxes the orthonormal constraint to derive a probabilistic clustering results. 
    Furthermore, we theoretically establish the equivalence between orthogonality and block diagonality under the doubly stochastic constraint. 
    By integrating \textbf{B}lock diagonal regularization into LoRD, expressed as the maximization of the Frobenius norm, we propose \textbf{B-LoRD}, which further enhances the clustering performance. 
    To ensure numerical solvability, we transform the non-convex doubly stochastic constraint into a linear convex constraint through the introduction of a class probability parameter. 
    We further theoretically demonstrate the gradient Lipschitz continuity of our LoRD and B-LoRD enables the proposal of a globally convergent projected gradient descent algorithm for their optimization.
    Extensive experiments validate the effectiveness of our approaches.
    The code is publicly available at \url{https://github.com/lwl-learning/LoRD}.
\end{abstract}
    
\begin{keywords}
    Graph-based clustering, low-rank, kerne $k$-means, block diagonal, doubly stochastic
\end{keywords}

\section{Introduction} \label{sec:intro}

Graph-based clustering \citep{schaeffer2007graph, berahmand2025comprehensive, xue2024comprehensive, kang2021structured, wu2022effective} stands as a foundational technique in data mining and machine learning, aiming to partition data points based on similarity among samples. In this study, we approach graph-based clustering through the lens of kernel $k$-means.

\subsection{Kernel $k$-means and its Variants}
Let $\{x_1, \dots, x_n\}$ be $n$ data points to be grouped into $k$ clusters $G_1, \dots, G_k$, and $S_{ij} = \kappa(x_i, x_j)$ be a symmetric similarity matrix defined by a kernel function $\kappa(x_i, x_j)$, e.g., $\kappa(x_i, x_j) = \exp(-\|x_i - x_j \|^2 / \sigma^2)$ for the Gaussian kernel.
Kernel $k$-means \citep{dhillon2004kernel} seeks to maximize intra-class similarity by partitioning data into clusters $G_1, \dots, G_k$, as expressed by
\begin{equation}
    \max_{G_1, \dots, G_k} \sum_{r=1}^{k} \frac{1}{n_r} \sum_{x_i, x_j \in G_r} S_{ij},
    \label{eq:kmeans_dis}
\end{equation}
where $n_r = |G_r|$ represents the size of $G_r$. To transform Eq. \eqref{eq:kmeans_dis} into matrix form, we introduce the class assignment matrix $V \in \R^{n \times k}$ with $V_{ij} = 1 / \sqrt{n_j}$ if $x_i \in G_j$ and zero otherwise.
Notably, the definition of $V$ aligns with  constraints $V \geq 0, VV^T 1_n = 1_n, V^T V = I_k$, where $1_n$ is an $n$-dimensional vector of ones, $I_k$ is the identity matrix of size $k$, $V \geq 0$ means $\forall i,j, V_{ij} \geq 0$.
Consequently, Eq. \eqref{eq:kmeans_dis} can be equivalently expressed as:
\begin{equation}
    \max_{V \geq 0} \mathrm{Tr}(V^T S V), \; \text{s.t.} \; VV^T 1_n = 1_n, V^T V = I_k.
    \label{eq:kmeans_tr}
\end{equation}
Another equivalent form of Eq. \eqref{eq:kmeans_tr} is written as \citep{ding2005equivalence}:
\begin{equation}
    \min_{V \geq 0} \|S - VV^T \|_F^2, \; \text{s.t.} \; VV^T 1_n = 1_n, V^T V = I_k,
    \label{eq:kmeans_fro}
\end{equation}
where $\| \cdot \|_F$ denotes the Frobenius norm.
In Eq. \eqref{eq:kmeans_tr} and Eq. \eqref{eq:kmeans_fro}, the low-rank constraint $V \in \R^{n \times k}$ and the nonnegative constraint $V \geq 0$ make $V$ class-indicative.
Specifically, $\hat{y_i} = \arg\max_{j} V_{ij}$ can be regarded as an estimation of the cluster index of $x_i$.
Meanwhile, the doubly stochastic constraint $VV^T 1_n = 1_n$ enables $V_{ij}$ to express the probability that $x_i$ belongs to cluster $G_j$, as will be discussed in Theorem \ref{theorem:proba}.
Moreover, the orthonormal constraint $V^T V = I_k$ makes $V$ more discriminatory.
We will demonstrate in Theorem \ref{theorem:blk-diag} that the orthogonality of $V$ is closely related to the block diagonality \citep{lu2018subspace} of $VV^T$.

However, kernel $k$-means poses an NP-hard problem \citep{aloise2009np}. By relaxing the constraints in Eq. \eqref{eq:kmeans_tr} to make it numerically tractable, a series of classic graph-based clustering methods are derived as follows.
%A lot of classic graph-based clustering methods are derived by relaxing the constraints in Eq. \eqref{eq:kmeans_tr} to make it numerically tractable, such as:

\begin{figure}[t]
    \centering
    \resizebox{\textwidth}{!}{
    \begin{tikzpicture}[
        box/.style={draw, thick, rounded corners, align=center, minimum width=3.3cm, minimum height=0.7cm, inner sep=3pt},
        arrow/.style={thick, -{Latex[scale=0.8]}},
        node distance=0.8cm and 4cm,
        font=\scriptsize\bfseries
    ]
    
    \definecolor{myBlue}{RGB}{21, 96, 130}
    \definecolor{myOrange}{RGB}{191, 79, 20}

    % Legends
    \draw[arrow] (-2cm, 2.25cm) -- (-1.2cm, 2.25cm) node [right=0.1cm] {relaxation};
    \draw[arrow, {Latex[scale=0.8]}-{Latex[scale=0.8]}] (-2cm, 2cm) -- (-1.2cm, 2cm) node [right=0.1cm] {equivalence};
    \draw[dashed, arrow] (1.3cm, 2.25cm) -- (2.1cm, 2.25cm) node [right=0.1cm] {approximation with condition};
    \draw[densely dashed, arrow, {Latex[scale=0.8]}-{Latex[scale=0.8]}] (1.3cm, 2cm) -- (2.1cm, 2cm) node [right=0.1cm] {equivalence with condition};
    \node[draw, dashed, text width=9cm, text height=0.5cm, yshift=0.1cm] at (2.5cm, 2cm) {};

    \node[draw, dashed, text width=4.7cm, text height=0.5cm, yshift=0.1cm] at (13.9cm, 2cm) {\shortstack{$\circled{1} \; \circled{2} \; \circled{3} \; \circled{4} \; \circled{5} \; \circled{6}$ \\ recommended reading sequence}};
    
    % Nodes for the diagram
    \node[box, draw=myBlue, label={[xshift=-0.1cm, color=myBlue] kernel $k$-means Eq. \eqref{eq:kmeans_tr}}] (kernel-Tr) {$\max\limits_{V \geq 0} \mathrm{Tr}(V^T S V)$ \\ s.t. $V V^T 1_n = 1_n, V^T V = I_k$};
    
    \node[box, right=of kernel-Tr, xshift=-1cm, label={Spectral Clustering Eq. \eqref{eq:SC}}] (spectral) {$\max\limits_{V^T V = I_k} \mathrm{Tr}(V^T S V)$};
    
    \node[box, draw=myOrange, text width=3.7cm, right=of spectral, label={[color=myOrange, yshift=-0.1cm]B-LoRD Eq. \eqref{eq:B-LoRD}}] (B-LoRD) {$\max\limits_{V \geq 0} \mathrm{Tr}(V^T S V) + \gamma \| V \|_F^2$\\ s.t. $V^T 1_n = \mu, V \mu = 1_n / n$};
    \node[box, draw=myOrange, text width=3.7cm, below=of B-LoRD, label={[xshift=-1.6cm, color=myOrange]LoRD Eq. \eqref{eq:LoRD}}] (LoRD) {$\min\limits_{V \geq 0} \| S - VV^T \|_F^2$\\ s.t. $V^T 1_n = \mu, V \mu = 1_n / n$};
    \node[box, draw=myOrange, text width=3.7cm, below=of LoRD, label={[xshift=-1.6cm, color=myOrange]\shortstack{Intermediate \\ model Eq. \eqref{eq:mdl_1}}}] (model1) { $\min\limits_{V \geq 0} \| S - VV^T \|_F^2$\\ s.t. $V V^T 1_n = 1_n$};
    
    \node[box, left=of LoRD, label={SymNMF Eq. \eqref{eq:SymNMF}}] (symnmf) {$\min\limits_{V \geq 0} \| S - VV^T \|_F^2$};
    \node[box, left=of model1, label={DSN Eq. \eqref{eq:DSN}}] (dsn) {$\min\limits_{Z \geq 0} \| S - Z \|_F^2$\\ s.t. $Z = Z^T, Z 1_n = 1_n$};
    
    \node[box, draw=myBlue, left=of dsn, xshift=1cm, label={[xshift=-1.1cm, color=myBlue] \shortstack{kernel $k$-means \\ Eq. \eqref{eq:kmeans_fro}}}] (kernel-Fro) {$\min\limits_{V \geq 0} \| S - VV^T \|_F^2$ \\ s.t. $V V^T 1_n = 1_n, V^T V = I_k$};

    \node[draw=myOrange, dashed, text width=3.9cm, text height=5.2cm, yshift=0.1cm, label={[color=myOrange]Our proposal}] at (LoRD) {};
    
    % Arrows between nodes
    \draw[arrow, draw=myBlue, {Latex[scale=0.8]}-{Latex[scale=0.8]}] (kernel-Tr) -- (kernel-Fro) node[right, midway] {$\circled{1}$};
    \draw[arrow, draw=myBlue] (kernel-Tr) -- (spectral) node[anchor=south, midway] {$\circled{2}$ keep $V^T V = I_k$};
    \draw[arrow, draw=myBlue] (kernel-Fro.north east) -- (symnmf.south west) node[anchor=south, midway, sloped] {$\circled{2}$ keep $V \geq 0$};
    \draw[arrow, draw=myBlue] (kernel-Fro) -- (dsn) node[anchor=south, midway] {$\circled{2}$ relax $V^T V = I_k$} node[anchor=north, midway] {and $Z = VV^T$};
    \draw[dashed, arrow, draw=myOrange] (B-LoRD) -- (spectral) node[anchor=south, midway] {$\circled{6}$ relax constraints} node[anchor=north, midway] {and $\gamma \geq -\lambda_{\min}(S)$};
    \draw[arrow, draw=myOrange] (LoRD) -- (symnmf) node[anchor=south, midway] {$\circled{6}$ keep $V \geq 0$};
    \draw[arrow, draw=myOrange] (model1.north west) -- (symnmf.south east) node[anchor=south, midway, sloped] {$\circled{6}$ keep $V \geq 0$};
    \draw[arrow, draw=myOrange] (model1) -- (dsn) node[anchor=south, midway] {$\circled{6}$ relax $Z = VV^T$};
    \draw[densely dashed, arrow, {Latex[scale=0.8]}-{Latex[scale=0.8]}, draw=myOrange] (model1) -- (LoRD) node[anchor=west, midway] {\shortstack{constraint reduction \\ (Theorem \ref{theorem:partition-space} and \ref{theorem:proba})}} node[anchor=east, midway] {$\circled{4}$};
    % \draw[dashed, arrow, draw=myOrange] (LoRD) -- (model1);
    \draw[dashed, arrow, draw=myOrange] (B-LoRD) -- (LoRD) node[anchor=west, midway] {\shortstack{block diagonality \\ (Theorem \ref{theorem:blk-diag})}} node[anchor=east, midway] {$\circled{5}$};

    \draw[arrow, draw=myBlue] (kernel-Fro.south) -- ([yshift=-0.6cm] kernel-Fro.south) -- ([xshift=5cm, yshift=-0.6cm] dsn.south) node[anchor=south, midway] {$\circled{3}$ relax $V^T V = I_k$} -- ([yshift=-0.6cm] model1.south) -- (model1.south);
    
    \draw[dashed, arrow, draw=myBlue] ([xshift=-0.3cm] kernel-Tr.north east) -- ([xshift=-0.3cm, yshift=0.6cm] kernel-Tr.north east) -- ([xshift=10cm, yshift=0.6cm] kernel-Tr.north east) node[anchor=south, midway] {$\circled{6}$ replace $V^T V = I_k$ with regularization $\|V \|_F^2$} -- ([xshift=0.5cm, yshift=0.6cm] B-LoRD.north west) -- ([xshift=0.5cm] B-LoRD.north west);
    \end{tikzpicture}
    }
    \caption{Schematic diagram of graph-based clustering methods. 
    $\circled{1}$: The models in Eq. \eqref{eq:kmeans_tr} and Eq. \eqref{eq:kmeans_fro} are equivalent.
    $\circled{2}$: SC, SymNMF and DSN are overly relaxed kernel $k$-means.
    $\circled{3}$: The model in Eq. \eqref{eq:mdl_1} only relaxes the least important orthonormal constraint $V^T V = I_k$.
    $\circled{4}$: The non-convex constraint $VV^T 1_n = 1$ can be reduced to a convex constraint $V^T 1_n = \mu, V\mu = 1_n /n$ by incorporating the class probability parameter $\mu$ (Theorems \ref{theorem:partition-space} and \ref{theorem:proba}), making our LoRD numerically solvable and interpretable in terms of probability.
    $\circled{5}$: The $k$-block diagonality of $VV^T$ can be adjusted by tuning $\gamma \in [-\lambda_{\max}(S), -\lambda_{\min}(S)]$ in B-LoRD.
    $\circled{6}$: kernel $k$-means, SC, SymNMF, and DSC can be seen as a relaxation or approximation of our LoRD and B-LoRD.
    }
    \label{fig:schematic}
\end{figure}
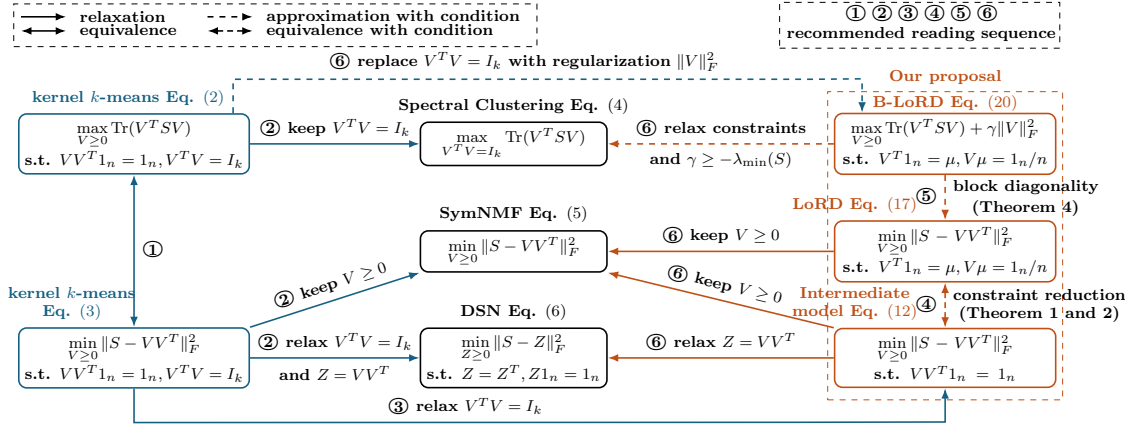

\textit{1) Spectral clustering (SC) \citep{von2007tutorial}.} 
SC only retains the orthonormal constraint in Eq. \eqref{eq:kmeans_tr}, i.e.,
\begin{equation}
    \max_{V^T V = I_k} \mathrm{Tr}(V^T S V).
    \label{eq:SC}
\end{equation}
The optimum of Eq. \eqref{eq:SC} is given by the eigenvectors of $S$ corresponding to the largest $k$ eigenvalues.
However, since the nonnegative constraint $V \geq 0$ is relaxed, $V$ cannot directly provide the clustering result.
Thus, SC requires post-processing to obtain the clustering result, such as performing $k$-means on $V$.

\textit{2) Symmetric non-negative matrix factorization (SymNMF) \citep{kuang2012symmetric, kuang2015symnmf}.}
In contrast to SC, SymNMF only retains the nonnegative constraint in Eq. \eqref{eq:kmeans_fro}, i.e.:
\begin{equation}
    \min_{V \geq 0} \|S - VV^T \|_F^2.
    \label{eq:SymNMF}
\end{equation}
SymNMF can obtain clustering results without post-processing.
However, the probabilistic interpretability and discriminability are compromised because the doubly stochastic constraint is relaxed.

\textit{3) Doubly stochastic normalization (DSN) \citep{zass2005unifying, zass2006doubly}.}
DSN relaxes the orthonormal constraint $V^T V = I_k$ and the low-rank constraint by parameterizing $Z = VV^T$ in Eq. \eqref{eq:kmeans_fro} to solve the following convex problem:
\begin{equation}
    \min_{Z \geq 0} \|S - Z \|_F^2, \; \text{s.t.} \; Z = Z^T, Z1_n = 1_n.
    \label{eq:DSN}
\end{equation}
The probabilistic interpretability of $Z$ is analyzed in \citep{zass2005unifying}.
However, the low-rank structure in $Z$ is relaxed, resulting in the need for post-processing to extract the clustering result from $Z$, such as performing SC on it.

In conclusion, SC, SymNMF, and DSN are essentially relaxations of the kernel $k$-means, and their relationships are schematically drawn in Fig. \ref{fig:schematic}.
However, the necessity for relaxation to ensure numerical tractability may limit the overall clustering performance of these methods.

Among the constraints in Eq. \eqref{eq:kmeans_tr} and Eq. \eqref{eq:kmeans_fro}, the orthonormal constraint $V^T V = I_k$ and the doubly stochastic constraint $VV^T 1_n = 1_n$ are the most difficult to handle in optimization.
From the perspective of clustering, the orthonormal constraint is considered the least crucial \citep{zass2005unifying, zass2006doubly}, because removing $V^T V = I_k$ only transforms the original problem from hard clustering to soft clustering.
In this paper, to handle the doubly stochastic constraint, we prove in Theorem \ref{theorem:partition-space} and Theorem \ref{theorem:proba} that $VV^T 1_n = 1_n$ can be reduced to a linear convex constraint $V^T 1_n = \mu, V\mu = 1_n / n$, where $\mu$ is a user-specified parameter associated with class prior probability.
This observation enables the optimization problem to be solved efficiently.
Moreover, we establish the connection between the orthogonality of $V$ and the block diagonality of $VV^T$ in Theorem \ref{theorem:blk-diag}, so that the relaxed orthonormal constraint can be remedied by adding a block diagonal regularization of $VV^T$, as introduced in the subsequent subsection.

%\subsection{Related Work}
\subsection{Block Diagonal Structure}
% As discussed above, the low-rank, nonnegative and doubly stochastic structrues are ideal clustering structrues to obtain a probabilistic clustering.
An ideal clustering structure for a similarity matrix $S \in \R^{n \times n}$ with $n$ data points is one that has exactly $k$ connected components, where $k$ is the number of clusters, and each connected component corresponds to a cluster.
Such an $S$ can be expressed as a $k$-block diagonal matrix \citep{feng2014robust, lu2018subspace} as follows:
\begin{equation}
    S = \begin{bmatrix}
        S_1 & 0 & \cdots & 0 \\
        0 & S_2 & \cdots & 0 \\
        \vdots & \vdots & \ddots & \vdots \\
        0 & 0 & \cdots & S_k \\
    \end{bmatrix},
    \; \text{where} \; S_i \in \R^{n_i \times n_i}.
    \label{eq:blk-diag}
\end{equation}
According to the spectral graph theorem \citep{von2007tutorial}, the number of connected components of $S$ is equal to the multiplicity $k$ of the eigenvalue $0$ of the Laplacian matrix $L_{S} = \mathrm{Diag}(S 1_n) - S$, where $\mathrm{Diag}(z)$ is a diagonal matrix with $z$ as its diagonal elements.
Building on this insight, the $k$-block diagonal structure of $S$ can be achieved by constraining $\mathrm{rank}(L_S) = n - k$ \citep{wang2016structured}.
However, the constraint is difficult to handle directly in optimization.
The most common approach is to relax it to a regularization $\|S \|_{\Boxed{k}}$ using Ky Fan's theorem \citep{wang2016structured, xie2017implicit, nie2014clustering}, i.e.,
\begin{equation}
    \|S \|_{\Boxed{k}} := \sum_{i=n-k+1}^{n} \lambda_i(L_S) 
    = \min_{\substack{0 \preceq W \preceq I_n \\ \mathrm{Tr}(W) = k}} \langle L_S, W \rangle,
\end{equation}
where $\lambda_i(L_S)$ is the $i$-th largest eigenvalue of $L_S$, and $0 \preceq W \preceq I_n$ means that $0 \leq \lambda_{\min}(W) \leq \lambda_{\max}(W) \leq 1$.
% $S$ is $k$-block diagonal when $\|S \|_{\Boxed{k}} = 0$.
Despite this relaxation, $\|S \|_{\Boxed{k}}$ requires an auxiliary variable $W$ to be alternatively optimized. 
Therefore, learning a $k$-block diagonal structure remains an optimization challenge.

%Based on the DSN Eq. \eqref{eq:DSN}, there are multiple ways to enhance the block diagonality of $Z$.
%For example, \citep{wang2016structured} incorporated the block diagonal regularization $\|Z \|_{\Boxed{k}}$.
When combining the doubly stochastic constraint $Z = Z^T$, $Z 1_n = 1_n$ in DSN, the block diagonality can also be boosted in different ways, possibly more easily than $\|Z \|_{\Boxed{k}}$.
For example, motivated by the Davis-Kahan theorem, two constraints $\sigma_{k}(Z) \geq c_1$ and $\sigma_{k+1}(Z) \leq c_2$ are introduced in \citep{park2017learning}, where $\sigma_{k}(Z)$ is the $k$-th largest singular value of $Z$, and $c_1, c_2 \in [0, 1]$ are hyper-parameters.
When $c_1$ is close to one, $\lambda_{n - k + 1}(L_{Z}) = 1 - \lambda_{k}(Z) \leq 1 - c_1^2$ becomes close to zero.
Thus, $\sigma_{k}(Z) \geq c_1$ can be seen as a relaxation of the $k$-block diagonal constraint $\mathrm{rank}(L_Z) \leq n - k$.
More recently, \citep{julien2022Learning} noted that if a matrix $Z$ is both doubly stochastic and idempotent (i.e., $Z^2 = Z$), then $Z$ is block diagonal.
Thus, the idempotent condition $Z^2 = Z$ is added as a constraint.
However, two common issues exist in the above methods: 
1) High computational complexity ($\mathcal{O}(n^3)$), as the low-rank structure of $Z$ is relaxed;
2) They only focus on enhancing block diagonality, as the regularization coefficient is nonnegative.

In this paper, we demonstrate in Theorem \ref{theorem:blk-diag} that when further combining the low-rank structure ($Z = VV^T$), the block diagonality of $VV^T$ can be enhanced (resp. weakened) by maximizing (resp. minimizing) $\|V \|_F^2$.

\subsection{Contributions} \label{sec:contribution}

The main contributions of this paper are summarized as follows.

\begin{enumerate}
    \item We propose \textbf{LoRD}, a \textbf{Lo}w-\textbf{R}ank \textbf{D}oubly stochastic clustering model through a systematic literature review (Fig. \ref{fig:schematic}), which only relaxes the least important orthonormal constraint $V^T V = I_k$ in Eq. \eqref{eq:kmeans_fro}.
Removing $V^T V = I_k$ from Eq. \eqref{eq:kmeans_fro} is equivalent to transforming hard clustering into probabilistic clustering (Theorem \ref{theorem:proba}).
Moreover, to ensure numerical solvability, we demonstrate that the quadratic non-convex stochastic constraint $VV^T 1_n = 1_n / n$ can be equivalently represented as a linear convex constraint $V^T 1_n = \mu, V \mu = 1_n/n$ (Theorem \ref{theorem:partition-space}), where $\mu$ is a user-specified parameter associated with class prior probability.

\item To further learn the $k$-block diagonal structure, we theoretically show that minimizing $\|VV^T \|_{\Boxed{k}}$ is equivalent to maximizing $\|V \|_F^2$ under the doubly stochastic constraint (Theorem \ref{theorem:blk-diag}).
Accordingly, we propose a low-rank block diagonal doubly stochastic clustering model, namely \textbf{B-LoRD}, which is formulated as a quadratic optimization problem with linear convex constraints.
Unlike existing methods, B-LoRD can enhance or weaken the block diagonal structure by tuning the hyper-parameter $\gamma$ from $-\lambda_{\max}(S)$ to $-\lambda_{\min}(S)$.

\item We propose an efficient projected gradient descent algorithm to solve LoRD and B-LoRD, with $\mathcal{O}(n^2 k)$ complexity per iteration, which can be reduced to $\mathcal{O}(n \log(n) k)$ by exploiting the sparsity of $S$.
This complexity is more efficient than existing DSN-based methods with $\mathcal{O}(n^3)$ complexity.
In addition, we theoretically prove that the objective functions of LoRD and B-LoRD are gradient Lipschitz continuous (Theorem \ref{theorem:lipschitz}), which enables the automatic setting of the step size of descent, and ensures their global convergence (Lemma \ref{theorem:convergence}).

\end{enumerate}

The remainder of this paper is organized as follows. Sec. \ref{sec:RW} briefly reviews existing graph-based clustering methods highly relevant to our methods. 
In Sec. \ref{sec:models}, we propose \textbf{Lo}w-\textbf{R}ank \textbf{D}oubly stochastic clustering (LoRD) and \textbf{B}lock diagonality regularized LoRD (B-LoRD), which are then numerically solved by an efficient yet effective projected gradient descent algorithm in Sec. \ref{sec:algorithm}.
In Sec. \ref{sec:experiments}, we experimentally evaluate the performance of our LoRD and B-LoRD on both synthetic and real-world datasets.
Finally, we conclude this paper in Sec. \ref{sec:conclusion}.

\section{Prior Graph-based Clustering Methods}\label{sec:RW}
In addition to the variants of kernel $k$-means and block diagonal structures enhanced methods in Sec. \ref{sec:intro}, in this section, we further briefly review prior graph-based clustering methods that are very relevant to our work.
We also refer readers to \citet{berahmand2025comprehensive, xue2024comprehensive} for a comprehensive review.   

Semi-definite programming (SDP) \citep{peng2007approximating, kulis2007fast} provides a convex relaxation of kernel $k$-means, formulated as:
\begin{equation}
    \max_{Z} \langle S, Z \rangle, \quad \text{s.t.} \quad Z \succeq 0, Z \geq 0, Z= Z^T, Z 1_n = 1_n, \mathrm{Tr}(Z) = k,
\end{equation}
where $Z \succeq 0$ denotes $Z$ is semi-positive defined.
The gap between SDP and kernel $k$-means lies in the idempotency constraint $Z = Z^2$ \citep{kulis2007fast}, which is relaxed in SDP.
Owing to its convexity, SDP enjoys well-established statistical guarantees \citep{giraud2019partial, chen2021cutoff}.
However, SDP suffers from its high complexity (i.e., $\mathcal{O}(n^{3.5})$ per iteration \citep{sun2020sdpnal+}), limiting its applicability.

To reduce the complexity of SDP, \citet{zhuang2024statistically} proposed NLR, which leverages a low-rank factorization $Z = UU^T$ and directly optimizes over $U \in \R^{n \times r}$ with $r \geq k$.
The NLR is formulated as:
\begin{equation}
    \max_{U \geq 0} \mathrm{Tr}(U^T S U), \; \text{s.t.} \; \|U \|_F^2 = k, UU^T 1_n = 1_n.
    \label{eq:stat_kmeans}
\end{equation}
Compared with Eq. \eqref{eq:kmeans_tr}, NLR can be viewed as a variant of kernel $k$-means where the orthogonality constraint $V^T V = I_k$ with $V \in \R^{n \times k}$ is relaxed to $\|V\|_F^2 = k$ and extended to $V \in \R^{n \times r}$.
Benefited from the low-rank structure, the per-iteration complexity of NLR is reduced to $\mathcal{O}(n^2 r t)$, where $t$ denotes the number of primal descent steps.
The primary optimization challenge of NLR arises from the doubly stochastic constraint $UU^T 1_n = 1_n$, which is quadratic and non-convex.
Consequently, the algorithm in \citet{zhuang2024statistically} requires careful tuning of both the step size and regularization coefficient, and the constraint $UU^T 1_n = 1_n$ is not strictly enforced.

The doubly stochastic constraint can instead be handled more effectively using the Majorization–Minimization (MM) framework.
For instance, the graph-based clustering method DCD \citep{yang2012clustering, yang2016low}  imposes both low-rank and doubly stochastic structures, formulated as:
\begin{equation}
    \min_{W \geq 0} D_{KL}(S || W \mathrm{Diag}^{-1}(W^T 1_n) W^T), \quad \text{s.t.} \quad W 1_k = 1_n,
\end{equation}
where $D_{KL}(\cdot | \cdot)$ denotes the KL divergence, and the learned similarity matrix $W \mathrm{Diag}^{-1}$ $(W^T 1_n) W^T$ is both low-rank and doubly stochastic.
The complexity per iteration of DCD is $\mathcal{O}(nkq)$, where $q$ denotes the sparsity of the input similarity matrix $S$ (typically $q = \mathcal{O}(\log n)$).

In contrast to DCD, we introduce a class prior probability vector $\mu \in \R^k$, which reduces the doubly stochastic constraint $VV^T 1_n = 1_n$ to the linear and convex constraints $U^T 1_n = \mu, U \mu = 1_n / n$.
By exploiting the Lipschitz continuity of the gradients in our formulation, we design a projected gradient descent algorithm that achieves the same $\mathcal{O}(nkq)$ complexity for sparse $S$.

\section{Proposed Methods} \label{sec:models}
\subsection{LoRD: Graph-Based Probabilistic Clustering}
%Unlike kernel $k$-means which seeks for a hard partition of data points, 
In contrast to the rigid partitions sought by kernel $k$-means, probabilistic clustering \citep{zass2005unifying} aims to determine the probability that $x_i$ belongs to a typical cluster, i.e., $\P(y_i = j | x_i), j = 1,\dots, k$.
However, the orthonormal constraint $V^T V = I_k$ in kernel $k$-means forces $V$ to be a hard clustering result, i.e., each row of $V$ has only one non-zero element. To this end, we relax $V^T V = I_k$ in Eq. \eqref{eq:kmeans_fro} to obtain a soft clustering result, i.e.,
\begin{equation}
    \min_{V \in \R^{n \times k}} \|S - VV^T \|_F^2, \; \text{s.t.} \; V \geq 0, VV^T 1_n = 1_n.
    \label{eq:mdl_1}
\end{equation}
Unlike SC, SymNMF, and DSN, Eq. \eqref{eq:mdl_1} solely relaxes the least crucial orthogonality constraint, which is necessary to obtain probabilistic clustering.
However, Eq. \eqref{eq:mdl_1} is difficult to optimize due to the non-convex quadratic constraint $VV^T 1_n = 1_n$.
To address this, we first express the constraint space of Eq. \eqref{eq:mdl_1} as $\Omega$:
\begin{equation}
    \Omega := \{V \in \R^{n \times k} \; | \; V \geq 0, V V^T 1_n = 1_n / n \},
\end{equation}
where we replace $VV^T 1_n = 1_n$ with $VV^T 1_n = 1_n / n$, which is equivalent to a scalar multiplication of $V$.
To reduce the quadratic constraint in $\Omega$, we denote $\mu = V^T 1_n$, so that $VV^T 1_n = 1_n / n$ is equivalently written as $V \mu = 1_n / n$.
In other words, we construct a subspace of $\Omega$ determined by $\mu$:
\begin{equation}
    \!\!\Omega(\mu) \!:=\! \{V \!\in\! \R^{n \times k} | V \geq 0, V^T 1_n \!=\! \mu, V \mu \!=\! 1_n / n \}.
\end{equation}
To ensure $\Omega(\mu)$ is a subspace of $\Omega$, we must have $\mu \geq 0$ and $\|\mu \|_2^2 = 1_n^T VV^T 1_n = 1$, indicating that $\mu$ should lie on the space of $\mathbb{S}_{+}^{k} = \{\mu \in \R^{k} \; | \; \mu \geq 0, \|\mu \|_2^2 = 1\}$, i.e., the nonnegative unit sphere embedded in $\R^{k}$.
The relationship between $\Omega(\mu)$ and $\Omega$ is formally stated in Theorem \ref{theorem:partition-space} below:
\begin{theorem} \label{theorem:partition-space}
    When $\mu$ is varied over $\mathbb{S}_{+}^{k}$, the family of space $\Omega(\mu)$ is a partition of $\Omega$, i.e.:
    \begin{itemize}
        \item $\forall \mu \in \mathbb{S}_{+}^{k}$, $\Omega(\mu)$ is non-empty, i.e., $\Omega(\mu) \ne \emptyset$.
        \item When $\mu$ is varied over $\mathbb{S}_{+}^{k}$, $\Omega$ is the union of $\Omega(\mu)$, i.e., $\Omega = \bigcup_{\mu \in \mathbb{S}_{+}^{k}} \Omega(\mu)$.
        \item $\forall \mu, \nu \in \mathbb{S}_{+}^{k}$ where $\mu \ne \nu$, the intersection of $\Omega(\mu)$ and $\Omega(\nu)$ is empty, i.e., $\Omega(\mu) \cap \Omega(\nu) = \emptyset$.
    \end{itemize}
\end{theorem}

\begin{proof}
    See Appendix \ref{prooftheorem1}.
\end{proof}

\begin{wrapfigure}{r}{0.3\textwidth}
    \centering
    \vspace{-12pt}
    \begin{tikzpicture}[scale=0.65]
        \definecolor{myBlue}{RGB}{21, 96, 130}
        \definecolor{myOrange}{RGB}{191, 79, 20}

        \draw[thick, color=myBlue, fill=myBlue!5] (0, 0) -- (2, 1) -- (2, -1.5) -- (-1.5, -1.5) -- (-2.5, 1) node [right=0.3]{$\Omega$} -- (0, 0); 

        \draw[thick, color=myOrange] (0, 0) -- (2, 0);
        \draw[thick, color=myOrange] (0, 0) -- (2, -1.5);
        \draw[thick, color=myOrange] (0, 0) -- (0.5, -1.5);
        \draw[thick, color=myOrange] (0, 0) -- (-1.5, -1.5);
        \draw[thick, color=myOrange] (0, 0) -- (-1.9, -0.5);

        \node[color=myOrange, font=\footnotesize] (mu) at (3.2, -0.5) {$\Omega(\mu)$};
        \draw[-latex] (4/3, -1) to[out=60, in=180] (mu.180);
    \end{tikzpicture}
\end{wrapfigure}
For a better understanding of Theorem \ref{theorem:partition-space}, the relationship between $\Omega$ and $\Omega(\mu)$ is schematically illustrated in the inset figure: any $V \in \Omega$ (non-convex, blue face) lies on an $\Omega(\mu)$ (convex, orange segment) that is determined by a certain $\mu \in \mathbb{S}_{+}^{k}$.
Therefore, it is natural to ask: \textit{What is the physical meaning of $\mu$, and which $\mu$ should we expect?}
The answers to these questions are given in Theorem \ref{theorem:proba} below.
\begin{theorem} \label{theorem:proba}
    Let $\P(c_j)$ be the prior probability of the $j$-th class, and $\P(x_i)$ the prior probability of $x_i$, assumed to be uniform, i.e., $\P(x_i) = 1 /n$.
    When $\mu = [\sqrt{\P(c_1)}, \dots, \sqrt{\P(c_k)}]^T$, any $V \in \Omega(\mu)$ can be expressed as:
    \begin{equation}
        V_{ij} = \frac{\P(y_i = j | x_i) \P(x_i)}{\sqrt{\P(c_j)}}.
    \end{equation}
    Thus, $V_{ij}$ is associated with the conditional probability of $x_i$ belonging to the $j$-th class $\P(y_i = j | x_i)$, indicating that $V$ corresponds to a probabilistic clustering result.

    Furthermore, the pairwise probability matrix can be recovered by $Z = n^2 V \mathrm{Diag}(\mu \odot \mu) V^T$, such that 
    \begin{equation}
        Z_{ij} = \P(y_i = y_j | x_i, x_j).
    \end{equation}
    In other words, $Z_{ij}$ describes the conditional probability of $x_i$ and $x_j$ belonging to the same class.
\end{theorem}

\begin{proof}
    See Appendix \ref{prooftheorem2}.
\end{proof}

In Theorem \ref{theorem:proba}, the $1 / n$ in the assumption $\P(x_i) = 1_n /n$ is drawn from $VV^T 1_n = 1_n / n$, indicating that if $\P(x_1),\dots, \P(x_n)$ are known, it may make sense to replace the doubly stochastic constraint with $VV^T 1_n = [\P(x_1), \dots, \P(x_n)]^T$.
More importantly, Theorem \ref{theorem:partition-space} and Theorem \ref{theorem:proba} state that when $\P(c_1), \dots, \P(c_k)$ are known, we expect the learned $V$ to lie on the $\Omega(\mu)$, where $\mu = [\sqrt{\P(c_1)}, \dots, \sqrt{\P(c_k)}]^T$, and \textbf{the constraint $V \in \Omega$ is equivalently reduced to $V \in \Omega(\mu)$.}
Building on this insight, we propose a low-rank doubly stochastic clustering (LoRD) model, which is formulated as 
\begin{equation}
    \min_{V \in \Omega(\mu)} \|S - VV^T \|_F^2.
    \label{eq:LoRD}
\end{equation}
Compared to Eq. \eqref{eq:mdl_1}, Eq. \eqref{eq:LoRD} is numerically solvable because $\Omega(\mu)$ is linear and convex.
In practice, as the class prior probability is generally unknown, we can simply set $\mu = [1 / \sqrt{k}, \dots, 1 / \sqrt{k}]^T$.
Our experiments in Sec. \ref{sec:robust} show that the proposed model is robust to the value of $\mu$.
\begin{remark}
    For LoRD in Eq. \eqref{eq:LoRD}, the optimization space $\Omega(\mu)$ only relaxes the least important orthonormal constraint $V^T V = I_k$ of kernel $k$-means in Eq. \eqref{eq:kmeans_fro}.
    For numerical solvability, we further reduce $\Omega$ to $\Omega(\mu)$ (Theorem \ref{theorem:partition-space}), where $\mu$ is a user-specified parameter associated with the prior probability of the class (Theorem \ref{theorem:proba}).
    As a result, LoRD can learn a probabilistic clustering result (Theorem \ref{theorem:proba}).
\end{remark}

\subsection{B-LoRD: Adjusting $k$-Block Diagonality of $VV^T$}
Although LoRD achieves probabilistic clustering, it remains unclear how to control the distribution of $\P(y_i = j | x_i)$ (sharp or uniform), which is closely related to the orthogonality of $V$:
when $V$ is fully orthogonal, we have the sharpest clustering probability $\P(y_i = j | x_i) = 1$ if $x_i \in G_j$ and zero otherwise;
when $V$ is least orthogonal, we have the uniform clustering probability $\P(y_i = j | x_i) = 1 / k, j=1,\dots, k$.
Interestingly, we find that the orthogonality of $V$ is equivalent to the $k$-block diagonality of $VV^T$ under the doubly stochastic constraint, as described in the following theorem:
% In LoRD, $VV^T$ is regarded as a similarity matrix with probabilistic interpretability.
% However, although $VV^T$ is low-rank, its block diagonality is not clear.
% To this end, we introduce $\|VV^T \|_{\Boxed{k}}$ as a regularization.
\begin{theorem} \label{theorem:blk-diag}
    For any $V \in \Omega$ (which naturally includes any $V \in \Omega(\mu)$), the following equality holds:
    \begin{equation}
        \|VV^T \|_{\Boxed{k}} = \frac{k}{n} - \|V \|_F^2.
    \end{equation}
    Specifically, the least $k$-block diagonal case ($\|VV^T \|_{\Boxed{k}}$ is maximized) occurs when $V \in \{1_n \mu^T / n \;|\; \mu \in \mathbb{S}_{+}^{k}\}$, and the fully $k$-block diagonal case ($\|VV^T \|_{\Boxed{k}}$ is minimized to zero) occurs when $V$ is orthogonal.
\end{theorem}

\begin{proof}
    See Appendix \ref{prooftheorem4}.
\end{proof}

% For one hand, Theorem \ref{theorem:blk-diag} states that the $k$-block diagonal regularization $\|VV^T \|_{\Boxed{k}}$ can be simplified as $\|V \|_F^2$.
Besides, the objective function in Eq. \eqref{eq:LoRD} is equivalent to 
\begin{equation}
     \|S - VV^T \|_F^2 = \|S \|_F^2 -2\mathrm{Tr}(V^T SV) + \|V V^T\|_F^2,
\end{equation}
where $\|S \|_F^2$ can be treated as a constant, and the role of minimizing $\|VV^T \|_F^2$ under the constraint $V \in \Omega(\mu)$ is to weaken the block diagonality of $VV^T$.
To see this, we have $\|VV^T \|_F^2 = \sum_{j=1}^{k} \sigma_j^4(V) \geq \frac{1}{n^2}$, where the lower bound is achieved when $V = 1_n \mu^T \!/ n$ with $\sigma(V) = [\frac{1}{\sqrt{n}}, 0, \dots, 0]^T\!$.

Motivated by the above analysis, we propose a low-rank block diagonal doubly stochastic clustering (B-LoRD) model (replace $\|VV^T \|_F^2$ in Eq. \eqref{eq:LoRD} with $-2\gamma \|V \|_F^2$):
\begin{equation}
    \max_{V \in \Omega(\mu)} \mathrm{Tr}(V^T SV) + \gamma \|V \|_F^2,
    \label{eq:B-LoRD}
\end{equation}
where $\gamma$ is a hyper-parameter that controls the block diagonality of $VV^T$.
Specifically, the objective function of Eq. \eqref{eq:B-LoRD} can be written as $\mathrm{Tr}(V^T (S + \gamma I_n) V)$, indicating that the value of $\gamma$ should lie in the range $[-\lambda_{\max}(S), -\lambda_{\min}(S)]$.
When $\gamma \leq -\lambda_{\max}(S)$, Eq. \eqref{eq:B-LoRD} becomes a convex optimization problem, and its global optimum is trivial: $V = 1_n \mu^T / n$.
When $\gamma = -\lambda_{\min}(S)$, we observe that the learned $V$ is almost orthogonal (see Fig. \ref{fig:S_gamma} for details), indicating that $\gamma$ is sufficiently large.
Note that $\gamma$ can be negative, which means the block diagonality of $VV^T$ is weakened.

\begin{remark}
    In the proposed LoRD, we reduce the orthogonal constraint $V^TV = I_k$ to ensure numerical solvability.
    In B-LoRD, the orthogonality of $V$ is controlable by adjusting $\gamma \in [-\lambda_{\max}(S), -\lambda_{\min}(S)]$.
\end{remark}

%Moreover, as detailed in Appendix \ref{sec:relation}, we analyze the relationship between our models and other graph-based clustering methods.

\subsection{Relation to Other Graph-Based Clustering Methods}

\indent 

\textbf{Kernel $k$-means.} Our model in Eq. \eqref{eq:mdl_1}, which is equivalent to LoRD in Eq. \eqref{eq:LoRD}, assumes the class prior probability $\mu$, relaxes only the least important orthogonality constraint $V^T V = I_k$ in Eq. \eqref{eq:kmeans_fro}.
Moreover, our B-LoRD in Eq. \eqref{eq:B-LoRD} replaces the orthonormal constraint $V^T V = I_k$ in Eq. \eqref{eq:kmeans_tr} with a regularization term $\gamma \|V \|_F^2$. Therefore, when $\gamma$ is sufficiently large (e.g., $\gamma \geq -\lambda_{\min}(S)$), our LoRD and B-LoRD can be regarded as relaxations and approximations of kernel $k$-means, respectively.

\textbf{SC.} SC can be interpreted as a relaxation of our B-LoRD in Eq. \eqref{eq:B-LoRD} when $\gamma$ is sufficiently large and the constraint $V \in \Omega(\mu)$ is relaxed. As a consequence, SC requires a post-processing step to obtain the final clustering result.

\textbf{SymNMF.} SymNMF can be thought of as a relaxation of Eq. \eqref{eq:mdl_1} and Eq. \eqref{eq:LoRD}, where the constraints $V \in \Omega$ and $V \in \Omega(\mu)$ are relaxed to $V \geq 0$. This relaxation leads to the loss of both probabilistic interpretability and discriminative capability.
%Consequently, the probabilistic interpretability and discriminability are lost.
As analyzed in Sec. \ref{sec:clustering}, minimizing the objective function of SymNMF does not significantly improve clustering performance.  In contrast, our LoRD and B-LoRD demonstrate meaningful performance gains.

\textbf{DSN.} Building upon our model in Eq. \eqref{eq:mdl_1}, DSN further parameterizes $Z = VV^T$ to obtain a convex optimization problem. However, this comes at the cost of requiring post-processing to extract the clustering results and incurs higher computational complexity, typically $\mathcal{O}(n^3)$, due to the relaxation of the low-rank structure. In comparison, our LoRD and B-LoRD achieve lower computational complexity of $\mathcal{O}(n^2 k)$.

%As compensation, the DSN-based methods require post-processing to extract the clustering result, and their computational complexities are generally high (e.g., $\mathcal{O}(n^3)$) due to the relaxation of the low-rank structure.
%As comparison, our LoRD and B-LoRD achieve $\mathcal{O}(n^2 k)$ complexities.

\textbf{GWL.} Additionally,  our B-LoRD in Eq. \eqref{eq:B-LoRD} can be reformulated as a Gromov-Wasserstein learning problem \citep{chowdhury2021generalized}, which is an optimal transport \citep{montesuma2024recent} based approach to clustering. A detailed analysis of this connection is provided in the following subsection.

\subsection{B-LoRD VS. Gromov-Wasserstein Learning in Optimal Transport} \label{sec:OT}
Optimal transport (OT) \citep{montesuma2024recent} has received considerable attention in the machine learning community, as it learns a transport plan $P$ over the joint probability space:
\begin{equation}
    \Pi(\alpha, \beta) := \{P \in \R^{n \times k} \; | \; P \geq 0, P1_k = \alpha, P^T 1_n = \beta \},
    \label{eq:def_Pi}
\end{equation}
where $\alpha\geq 0$ and $\beta\geq 0$ are marginal probabilities satisfying $\alpha^T 1_n = \beta^T 1_k = 1$.

Gromov-Wasserstein learning (GWL) \citep{xu2019scalable, chowdhury2021generalized, van2024distributional} is an OT-based approach to graph partitioning, which solves 
\begin{equation}
    \min_{P \in \Pi(\alpha, \beta)} \sum_{i, j = 1}^{n} \sum_{a, b = 1}^{k} \ell(S_{ij}, C_{ab}) P_{ia} P_{jb},
\end{equation}
where $S \in \R^{n \times n}$ is the source graph describing the similarities between samples, $C \in \R^{k \times k}$ is the target graph describing the similarities between clusters, and $\ell$ is a loss function.

Let $D = \mathrm{Diag}(\mu)$, for any $V \in \Omega(\mu)$ such that $V \geq 0, V \mu = 1_n / n, V^T 1_n = \mu$, we have $P = VD \in \Pi(1_n / n, \mu \odot \mu)$, where $\odot$ denotes the Hadamard product.
By parameterizing $V = P D^{-1}$, our B-LoRD becomes
\begin{equation}
    \max_{P \in \Pi(1_n / n, \mu \odot \mu)} \mathrm{Tr}(D^{-1} P^T S P D^{-1}) + \gamma \|P D^{-1} \|_F^2.
    \label{eq:OT-B-LoRD}
\end{equation}
Interestingly, Eq. \eqref{eq:OT-B-LoRD} is mathematically similar to the GWL.
To demonstrate this, Eq. \eqref{eq:OT-B-LoRD} can be transformed into
\begin{equation}
    \min_{P \in \Pi(1_n / n, \mu \odot \mu)} -\sum_{i,j=1}^{n} \sum_{a=1}^{k} (S + \gamma I_n)_{ij} \mu_a^{-2} P_{ia} P_{ja},
\end{equation}
where $D^{-1}$ is regarded as the target graph in which each cluster is only similar to itself, and the loss function is $\ell(S_{ij}, D_{aa}) = - (S + \gamma I_n)_{ij} \mu_a^{-2}$.

\section{Numerical Optimization} \label{sec:algorithm}
\subsection{Optimization Framework}
We propose a projected gradient descent method to solve Eq. \eqref{eq:LoRD} and Eq. \eqref{eq:B-LoRD}.
Let $f_1(V)$ and $-f_2(V)$ be the objective functions of Eq. \eqref{eq:LoRD} and Eq. \eqref{eq:B-LoRD}, respectively.
Note that we transform Eq. \eqref{eq:B-LoRD} into the problem of minimizing $f_2(V)$ for a consistent description with Eq. \eqref{eq:LoRD}.
The gradients of $f_1(V)$ and $f_2(V)$ are
\begin{equation}
    \nabla_1(V) = 4 (VV^T - S) V \quad \text{and} \quad \nabla_2(V) = -2(SV + \gamma V),
    \label{eq:gradient}
\end{equation}
respectively, and they have the following property. 
\begin{theorem} \label{theorem:lipschitz}
    For $V \in \Omega$ (naturally includes $V \in \Omega(\mu)$), $\nabla_1$ and $\nabla_2$ are Lipschitz continuous, where the Lipschitz constant $L_1$ and $L_2$ are:
    \begin{equation}
        L_1 = 4\left(3 / n + \|S \|_{\mathrm{op}} \right) \quad \text{and} \quad  L_2 = 2\left\|S + \gamma I_n \right\|_{\mathrm{op}},
    \label{eq:lipschitz}
    \end{equation}
    where $\|\cdot \|_{\mathrm{op}}$ denotes the operator norm, i.e., the largest singular value of a matrix.
\end{theorem}
\begin{proof}
    See Appendix \ref{prooftheorem5}.
\end{proof}
In the remainder of this paper, we omit the subscripts of $f, \nabla$, and $L$ if they are clear in context.
Since $\nabla$ is Lipschitz continuous, the step size of the projected descent can be automatically set to $1 / L$, leading to the update formula to solve Eq. \eqref{eq:LoRD} and Eq. \eqref{eq:B-LoRD} at the $t$-th iteration as
\begin{equation}
V^{t+1} = \proj_{\Omega(\mu)}(V^{t} - \nabla(V^{t}) / L),
\label{eq:pgd_update}
\end{equation}
where $\proj_{\Omega(\mu)}$ is the orthogonal projector onto $\Omega(\mu)$:
\begin{equation}
    \proj_{\Omega(\mu)}(U) = \mathop{\arg\min}_{V \in \Omega(\mu)} \|V - U \|_F^2,
    \label{eq:def_proj}
\end{equation}
which can be computed using the Dykstra algorithm \citep{boyle1986method} introduced in the next subsection.
Note that $\proj_{\Omega(\mu)}(U)$ is well defined, that is, the optimization problem in Eq. \eqref{eq:def_proj} has a unique optimum since $\Omega(\mu)$ is convex.

Moreover, to initialize $V^{0} \in \Omega(\mu)$, we employ the Sinkhorn-Knopp algorithm \citep{sinkhorn1964relationship}, such that for any $U \in \R^{n \times k}$, $\text{Sinkhorn}(U, \mu) \in \Omega(\mu)$.
The explanation of the Sinkhorn-Knopp algorithm is provided in Sec. \ref{sec:init_alg}.

Alg. \ref{alg:main} summarizes the overall projected gradient descent algorithm for solving LoRD in Eq. \eqref{eq:LoRD} and B-LoRD in Eq. \eqref{eq:B-LoRD}, where we repeat Eq. \eqref{eq:pgd_update} until either the maximum iteration count is reached or $\|V^{t+1} - V^{t}\|_F / \|V \|_F \leq \delta_v$.
In Alg. \ref{alg:main}, $\text{rand}(n, k)$ returns a random $n\times k$ matrix with entries in $[0, 1]$.
% In LoRD, the scale of $S$ affects the result of $V$, so the normalization step $S \leftarrow S / (1_n^T S 1_n)$ is necessary.
In B-LoRD, instead of tuning the hyper-parameter $\gamma \in [-\lambda_{\max}(S), -\lambda_{\min}(S)]$, we use $\tau \in [0, 1]$ to calculate $\gamma = -\lambda_{\max}(S) + \tau (\lambda_{\max}(S) - \lambda_{\min}(S))$ for convenience.
\begin{algorithm}[t]
    \caption{Projected Gradient Descent Algorithm for Solving LoRD in Eq. \eqref{eq:LoRD} and B-LoRD in Eq. \eqref{eq:B-LoRD}}
    \label{alg:main}
    \KwIn{$S \in \R^{n \times n}$, $\mu \in \mathbb{S}_{+}^{k}$, $\tau \in [0, 1]$, $t_{\max} = 4000$, $\delta_v = 10^{-4}$}
    Initialize $V^{0} = \text{Sinkhorn}(\text{rand}(n, k), \mu)$ , $t=0$ \Comment*[r]{See Alg.\ref{alg:sinkhorn} for details}
    $\gamma = -\lambda_{\max}(S) + \tau (\lambda_{\max}(S) - \lambda_{\min}(S))$ \\
    Calculate Lipschitz constant $L$ according to Eq. \eqref{eq:lipschitz} \\
    
    \Repeat{$t = t_{\max}$ {\bfseries or} $\|V^{t+1} - V^{t} \|_F / \|V^t \|_F \leq \delta_v$}{
        $V^{t+1} = \proj_{\Omega(\mu)}(V^{t} - \nabla(V^t) / L)$ \Comment*[r]{See Alg.\ref{alg:Dykstra} for details}
        $t = t + 1$ \\
    }
\end{algorithm}

\subsection{Projection onto $\Omega(\mu)$} \label{sec:proj_alg}
In Alg. \ref{alg:main}, the projection onto $\Omega(\mu)$ is a crucial step, but the closed-form expression for $\proj_{\Omega(\mu)}(U)$ is difficult to derive.
To this end, we adopt the Dykstra algorithm \citep{boyle1986method} to compute $\proj_{\Omega(\mu)}(U)$, which is a powerful tool for solving projections onto the intersection of convex sets, provided that the projection onto each convex set can be easily computed.
Indeed, $\Omega(\mu)$ can be seen as the intersection of two convex sets: $\R_{+}^{n \times k} := \{V \in \R^{n \times k} | V \geq 0 \}$ and $\Omega_0(\mu) := \{V \in \R^{n \times k} | V^T 1_n = \mu, V \mu = 1_n/n \}$, each of these admits a closed-form projection operator, summarized in Lemma \ref{lemma:proj}.
\begin{lemma} \label{lemma:proj}
    For any $U \in \R^{n \times k}$, we have:
    \begin{equation}
        \proj_{\R_{+}^{n \times k}}(U) = \max(U, 0)
        \quad \mathrm{and} \quad
        \proj_{\Omega_0(\mu)}(U) = U + \tfrac{1_n^T U \mu + 1}{n} 1_n \mu^T - \tfrac{1_n 1_n^T}{n} U - U \mu \mu^T.
    \end{equation}
\end{lemma}
\begin{proof}
    See Appendix \ref{prooflemma6}.
\end{proof}
Based on this, we employ a modified Dykstra algorithm in Alg. \ref{alg:Dykstra}, where $b_{\max}$ and $\delta_d$ are predefined maximum iteration count and convergence tolerance, respectively, and $\min(V)$ represents the minimal element of $V$.
In Alg. \ref{alg:Dykstra}, we use an adaptive step strategy \citep{Combettes2009ProximalSM} to accelerate convergence: as the iteration count $b$ grows from $0$ to $\infty$, the step size $\beta$ grows from $1$ to $2$.
\begin{comment}
In Alg. \ref{alg:Dykstra}, $Z$ is a deflection component as described in \citep{boyle1986method}. 
This component is redundant when the corresponding convex set is a linear subspace, e.g., $\Omega_0(\mu)$.
When all convex sets are linear subspaces, the Dykstra algorithm is equivalent to the well-known Von-Neumann successive projection method \citep{ginat2018method}.
\end{comment}
\begin{algorithm}[h]
    \caption{Modified Dykstra Algorithm for Solving $\proj_{\Omega(\mu)}(U)$}
    \label{alg:Dykstra}
    \KwIn{$U \in \R^{n \times k}$, $\mu \in \mathbb{S}_{+}^{k}$, $b_{\max} = 1000$, $\delta_d = 10^{-5}$}
    \KwOut{$V$}
    Initialize $V = U, \;Z = \text{zeros}(n, k), \; b=0, \; \alpha_b = 1$ \\
    
    \Repeat{$b = b_{\max}$ {\bfseries or} $-\min(V) \leq \delta_d \min(\max(\mu), 1 / (n \min(\mu)))$}{
        $b = b + 1$ \\
        $\alpha_{b+1} = \frac{1}{2} \left(1 + \sqrt{4 \alpha_{b}^2 + 1} \right)$ \\
        $\beta = 1 + (1 - \alpha_{b}) / \alpha_{b+1}$ \\
        $Y = (1 - \beta) V + \beta \max(V - Z, 0)$ \\
        $Z = Y - V + Z$ \\
        $V = (1 - \beta) Y + \beta \left[ Y + \tfrac{1_n^T Y \mu}{n} 1_n \mu^T - \tfrac{1_n 1_n^T}{n} Y - Y \mu \mu^T \right]$ \\
    }
    $V = \max(V, 0)$
\end{algorithm}

\subsection{Initialization Method} \label{sec:init_alg}

Given a strictly positive matrix $U \in \R^{n \times k}$, the Sinkhorn-Knopp algorithm \citep{sinkhorn1964relationship} seeks two diagonal matrices $D_r, D_l$ such that $D_r U D_l \in \Pi(\alpha, \beta)$ (the definition of $\Pi(\alpha, \beta)$ is given in \eqref{eq:def_Pi}).
Motivated by the relationship between $\Omega(\mu)$ and $\Pi(1_n / n, \mu \odot \mu)$, we first apply the Sinkhorn-Knopp algorithm to generate a random $P \in \Pi(1_n / n, \mu \odot \mu)$, and then normalize $V^{0} = P \mathrm{Diag}^{-1}(\mu)$ to obtain a random $V^{0} \in \Omega(\mu)$.
The pseudocode in the MATLAB syntax is provided in Alg. \ref{alg:sinkhorn}.

\begin{algorithm}[t]
    \caption{Sinkhorn-Knopp Algorithm for Normalizing $U \in \R^{n \times k}$ into $\Omega(\mu)$}
    \label{alg:sinkhorn}
    \KwIn{$U \in \R^{n \times k}$, $\mu \in \mathbb{S}_{+}^{k}$, $s_{\max} = 1000$, $\delta_s = 10^{-16}$}
    \KwOut{$V$}
    Initialize $\ell = 1_n, \; s=0$ \\
    $P = \max(U \odot \mu^T, 10^{-20})$ \Comment*[r]{Ensure $P$ is strict positive.}
    \Repeat{$s = s_{\max}$ \textbf{or} $\max(\max(|(P^T \ell) \odot r - \mu \odot \mu|), \max(|(P r) \odot \ell - 1_n / n|)) \leq \delta_s$}{
        $s = s + 1$ \\
        $r = (\mu \odot \mu) \oslash (P^T\ell)$ \\
        $\ell = (1_n / n) \oslash (P r)$ \\
    }
    $P = \mathrm{Diag}(l) \odot P \mathrm{Diag}(r)$ \\
    $V = P \mathrm{Diag}^{-1}(\mu)$ \\
\end{algorithm}

%\subsection{Why Dykstra Algorithm Is not Used for Initialization?} 
\subsection{Why not the Dykstra Algorithm for Initialization?}
\label{sec:add_init}
One may wonder why the Dykstra algorithm is not used to initialize $V^{0}$, since it is already available for projecting any $U \in \R^{n \times k}$ onto $\Omega(\mu)$. 
The key reason is that the objective functions of LoRD and B-LoRD are non-convex, which makes the quality of initialization crucial. 
To mitigate the risk of converging to poor stationary points, we employ multiple initializations $V^{0} \in \Omega(\mu)$. 
Ideally, these initializations would be drawn uniformly from $\Omega(\mu)$.
However, sampling uniformly from the set of doubly stochastic matrices remains an open problem \citep{cappellini2009random}.

Empirically, we found that Dykstra-based initializations often converge to poor stationary points, characterized by high objective values and weak clustering performance. 
In contrast, Sinkhorn–Knopp initializations are more effective at avoiding such outcomes.
This phenomenon can be explained by a geometric argument. 
Consider the simplified case $n=2$, $k=1$, where $\Omega(\mu)$ reduces to the two-dimensional simplex $\Delta_{2} = \{v \in \R^{2} \mid v \geq 0, v(1) + v(2) = 0.5\}$.

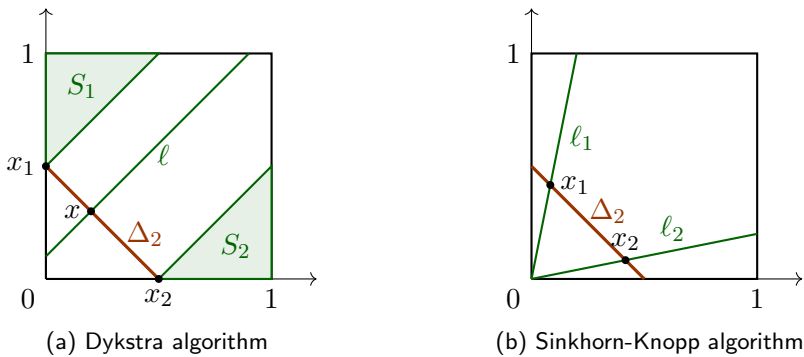
\begin{figure}[htbp]
    \centering
    \definecolor{edgeColor}{RGB}{153, 51, 0}
    \definecolor{lineColor}{RGB}{0, 100, 0}
    % \definecolor{lineColor}{RGB}{0, 51, 204}  % blue color

    \subfloat[Dykstra algorithm]{
        \begin{tikzpicture}[scale=3]
            \draw[->] (0, 0) -- (1.2, 0); 
            \draw[->] (0, 0) -- (0, 1.2);
            \draw[thick] (0, 0) node [below left]{$0$} -- (0, 1) node [left]{$1$} -- (1, 1) -- (1, 0) node [below]{$1$} -- (0, 0);
        
            \draw[very thick, color=edgeColor] (0, 0.5) -- (0.5, 0);
            \draw[thick, color=lineColor] (0, 0.1) -- (0.9, 1);
            \draw[thick, color=lineColor, fill=lineColor!10] (0, 0.5) -- (0.5, 1) -- (0, 1) -- (0, 0.5);
            \draw[thick, color=lineColor, fill=lineColor!10] (0.5, 0) -- (1, 0.5) -- (1, 0) -- (0.5, 0);
        
            % labels
            \fill (0, 0.5) circle (0.5pt) node [left]{$x_1$};
            \fill (0.5, 0) circle (0.5pt) node [below]{$x_2$};
            \fill (0.2, 0.3) circle (0.5pt) node [left]{$x$};
            \fill (0, 1) node [below right=0.2, color=lineColor]{$S_1$};
            \fill (1, 0) node [above left=0.2, color=lineColor]{$S_2$};
            \fill (0.45, 0.55) node [right, color=lineColor]{$\ell$};
            \fill (0.3, 0.2) node [right=1pt, color=edgeColor]{$\Delta_2$};
        \end{tikzpicture}
        \label{fig:init_Dykstra}
        }
    \hspace{2cm}
    \subfloat[Sinkhorn-Knopp algorithm]{
        \begin{tikzpicture}[scale=3]
            \draw[->] (0, 0) -- (1.2, 0); 
            \draw[->] (0, 0) -- (0, 1.2);
            \draw[thick] (0, 0) node [below left]{$0$} -- (0, 1) node [left]{$1$} -- (1, 1) -- (1, 0) node [below]{$1$} -- (0, 0);
        
            \draw[very thick, color=edgeColor] (0, 0.5) -- (0.5, 0);
            \draw[thick, color=lineColor] (0, 0) -- (1, 0.2);
            \draw[thick, color=lineColor] (0, 0) -- (0.2, 1);
        
            % labels
            \fill (1/12, 5/12) circle (0.5pt) node [right]{$x_1$};
            \fill (5/12, 1/12) circle (0.5pt) node [above]{$x_2$};
        
            \fill (1/8, 5/8) node [right, color=lineColor]{$\ell_1$};
            \fill (5/8, 1/8) node [above, color=lineColor]{$\ell_2$};
            \fill (0.2, 0.3) node [right=1pt, color=edgeColor]{$\Delta_2$};
        \end{tikzpicture}
        \label{fig:init_Sinkhorn}
    }
    \caption{The schematic diagram of the initialization method.}
    \label{fig:init}
\end{figure}

In Fig. \ref{fig:init_Dykstra}, let $x$ be an interior point of $\Delta_{2}$, i.e., $x > 0$ with $x(1)+x(2)=0.5$, and let $x_1=(0,0.5)$ and $x_2=(0.5,0)$ denote boundary points.
When initialization is based on the Dykstra algorithm, the probability density of sampling an interior point $x$ is proportional to the length of the line segment $\ell$ passing through $x$ with slope $1$, whereas the probability density of sampling boundary points $x_1$ and $x_2$ is proportional to the area of the regions $S_1$ and $S_2$, respectively.
Consequently, Dykstra-based initialization is biased toward boundary points.

In contrast, Fig. \ref{fig:init_Sinkhorn} illustrates that Sinkhorn–Knopp initialization produces a probability density at $x_1, x_2 \in \Delta_2$ proportional to the lengths of line segments $\ell_1$ and $\ell_2$ passing through $(0,0)$, thereby generating samples that are closer to uniform over $\Delta_2$.

\subsection{Convergence Analysis}
The convergence of Alg. \ref{alg:main} is described as follows.
% \begin{lemma}[{\citep[Theorem 10.15]{Beck2017FirstOrderMI}}] \label{theorem:convergence}
\begin{lemma} \label{theorem:convergence}
    \emph{\citep[Theorem 10.15]{Beck2017FirstOrderMI}.}
    Suppose $f(V)$ is gradient $L$-Lipschitz continuous, and $\Omega(\mu) \subseteq \R^{n \times k}$ is closed, convex and nonempty.
    Let $V^{\ast}$ be a global optimum of Eq. \eqref{eq:LoRD} or Eq. \eqref{eq:B-LoRD}.
    At the $t$-th iteration of Alg. \ref{alg:main}, the following inequality holds:
    \begin{equation}
        \min_{0 \leq i \leq t} \|V^{i+1} - V^{i} \|_F \leq \sqrt{\frac{2}{L} \frac{f(V^{0}) - f(V^{\ast})}{t + 1}}.
    \end{equation}
\end{lemma}
\begin{proof}
    See Appendix \ref{prooflemma7}.
\end{proof}
Lemma \ref{theorem:convergence} states the convergence condition $\tfrac{\|V^{t+1} - V^{t} \|_F}{\|V^t \|_F}$ $\leq \delta_v$ is always satisfied when $t$ is sufficiently large.

\subsection{Complexity Analysis} \label{sec:complexity}
Under the general setting of graph-based clustering, i.e., $S$ is an $n \times n$ matrix without any special structure, the computation of $\nabla(V)$ incurs a complexity of $\mathcal{O}(n^2 k)$.
To solve $\proj_{\Omega(\mu)}(U)$, the calculations of $Y, Z, V$ in Alg. \ref{alg:Dykstra} require $\mathcal{O}(n k)$ complexity per iteration.
Consequently, the overall complexity of Alg. \ref{alg:Dykstra} is $\mathcal{O}(n k b_{\mathrm{avg}})$, where $b_{\mathrm{avg}}$ is the average number of iterations of Alg. \ref{alg:Dykstra}.
In our experiments, $b_{\mathrm{avg}}$ is typically around $50$ in most cases  (see Appendix \ref{sec:add_bavg}).

To avoid $\mathcal{O}(n^2)$ complexity in practice, $S$ is commonly constructed as a sparse $q$-nearest neighbor ($q$-NN) graph \citep{hou2022progressive,park2017learning, wang2016structured}, with $q$ set to $\lfloor \log_2(n) \rfloor + 1$ as recommended by \citep{von2007tutorial}.
Under this configuration, the computation of $\nabla(V)$ requires  only $\mathcal{O}(n \log(n) k)$ complexity. The memory requirements for storing $S$ and $V$ are $\mathcal{O}(n \log n)$ and $\mathcal{O}(nk)$, respectively. Therefore, Alg. \ref{alg:main} has a per-iteration time complexity of $\mathcal{O}(n k (\log n + b_{\mathrm{avg}}))$   and a total memory complexity of $\mathcal{O}(n(\log n + k))$.
Moreover, Alg. \ref{alg:main} exclusively involves matrix multiplication operations, which ensures strong GPU compatibility and scalability for large-scale datasets.

%only involves matrix product operations, which enables well GPU compatibility and scalability for large-scale datasets.

Compared to the DSN-based block diagonal enhancement methods \citep{wang2016structured, park2017learning, julien2022Learning}, which exhibit $\mathcal{O}(n^3)$ complexity, our Alg. \ref{alg:main} is significantly more efficient, benefiting from the low-rank structure of $VV^T$.

\section{Experiments} \label{sec:experiments}
\subsection{Experimental Settings} \label{sec:exp_set}

\indent

\textbf{Datasets.}
We adopted 12 datasets as described in Table \ref{tab:datasets}. 
As our methods require the input of a prior class probability $\mu$, we divide the datasets into six class-balanced datasets and six class-imbalanced datasets with varying imbalance rates (IBR) defined in Eq. \eqref{eq:IBR} for better analysis.
\begin{table}[htbp]
    \centering
    \tabcolsep=2mm
    \caption{Descriptions of Datasets}
    \begin{scriptsize}
    \begin{tabular}{cccccc}
    \toprule
    Code    & Dataset     & Dimension         & \# Sample ($n$) & \# Cluster ($k$) & IBR \\
    \midrule
    D1      & \href{http://www.cad.zju.edu.cn/home/dengcai/Data/Yale/Yale_32x32.mat}{YaleB}     & $32 \times 32$    & $165$     & $15$  & $0$ \\
    D2      & \href{http://www.cad.zju.edu.cn/home/dengcai/Data/ORL/ORL_32x32.mat}{ORL}         & $32 \times 32$    & $400$     & $40$  & $0$ \\
    D3      & \href{https://archive.ics.uci.edu/static/public/139/synthetic+control+chart+time+series.zip}{CHART}       & $60$              & $600$     & $6$   & $0$ \\
    D4      & \href{http://www.cad.zju.edu.cn/home/dengcai/Data/USPS/USPS.mat}{USPS-1000}   & $16 \times 16$    & $1000$    & $10$  & $0$ \\
    D5      & \href{http://www.cad.zju.edu.cn/home/dengcai/Data/MLData.html}{Isolet}      & $617$             & $7797$    & $26$  & $0$ \\
    D6      & \href{http://www.cad.zju.edu.cn/home/dengcai/Data/COIL20/COIL100.mat}{COIL100}     & $32 \times 32$    & $7200$    & $100$  & $0$ \\
    D7      & \href{https://archive.ics.uci.edu/dataset/178/semeion+handwritten+digit}{Semeion}     & $16 \times 16$    & $1593$    & $10$  & $3 \times 10^{-5}$ \\
    D8      & \href{http://www.cad.zju.edu.cn/home/dengcai/Data/MNIST/Orig.mat}{MNIST}       & $28 \times 28$    & $70000$   & $10$  & $0.0006$ \\
    D9      & \href{http://www.cad.zju.edu.cn/home/dengcai/Data/MNIST/2k2k.mat}{MNIST-2000}  & $28 \times 28$    & $2000$    & $10$  & $0.0014$ \\
    D10     & \href{https://archive.ics.uci.edu/dataset/109/wine}{Wine}        & $13$              & $178$     & $3$   & $0.0114$ \\
    D11     & \href{https://archive.ics.uci.edu/dataset/110/yeast}{Yeast}       & $8$               & $1484$    & $10$  & $0.2503$ \\
    D12     & \href{https://archive.ics.uci.edu/dataset/39/ecoli}{Ecoli}       & $7$               & $336$     & $8$   & $0.2704$ \\
    \bottomrule
    \end{tabular}
    \end{scriptsize}
    \label{tab:datasets}
\end{table}

\textbf{Methods under comparison.}
We compared the proposed methods with the following five types of methods:

\noindent Kernel $k$-means-based methods:
\begin{itemize}
    \item Kernel $k$-means (KKM) \citep{dhillon2004kernel}: Approximately solves the kernel $k$-means problem in Eq. \eqref{eq:kmeans_tr} with multiple random restarts.
    \item Global kernel $k$-means (GKKM) \citep{tzortzis2009global}: A deterministic algorithm for solving KKM that uses an incremental approach to obtain clustering results, making it more likely to avoid poor local minima and to find a near-optimal solution.
\end{itemize}
Spectral clustering (SC)-based methods:
\begin{itemize}
    \item Spectral clustering (SC) (Alg. 3 in \citep{von2007tutorial}): A relaxation of kernel $k$-means that only keeps the orthogonality constraint.
    \item Normalized Cut (NCut) \citep{shi2000normalized}: A variant of SC that transforms graph partitioning into the problem of solving the eigenvectors of the normalized graph Laplacian matrix to achieve optimal segmentation by minimizing inter-class similarity and maximizing intra-class similarity.
    \item Spectral rotation (SR) \citep{huang2013spectral}: An improvement over SC. Instead of post-processing via $k$-means, SR imposes an additional orthonormal constraint to better approximate the optimal continuous solution.
    \item Discrete and balanced spectral clustering (DBSC) \citep{wang2023discrete}: An improvement over SC, which jointly learns the spectral factor and clustering result, with an adjustable balance rate for clusters.
    \item Direct spectral clustering (DirectSC) \citep{nie2024novel}: An improvement over SC, which adaptively learns an improved similarity graph as well as the corresponding spectral factor from an initial low-quality similarity graph. Both the learned similarity graph and spectral factor can be used to directly obtain the final clustering result.
\end{itemize}
SymNMF-based methods:
\begin{itemize}
    \item SymNMF \citep{kuang2012symmetric, kuang2015symnmf}: A relaxation of kernel $k$-means that only keeps the nonnegative constraint. The multiplicative update algorithm described in \citep{long2007relational} is applied to solve Eq. \eqref{eq:SymNMF}.
    \item PHALS \citep{hou2022progressive}: An efficient algorithm for solving SymNMF.
    \item Self-supervised SymNMF (S$^{3}$NMF) \citep{jia2021self}: Progressively boosts the quality from an initial low-quality similarity matrix by combining multiple class assignment matrices.
    \item NLR \citep{zhuang2024statistically}: A non-convex Burer-Monteiro factorization approach for solving the (kernel) $k$-means problem.
\end{itemize}
Doubly stochastic normalization (DSN)-based methods:
\begin{itemize}
    \item Doubly stochastic normalization (DSN) \citep{zass2006doubly}: A relaxation of kernel $k$-means that relaxes the orthogonality constraint and the low-rank structure.
    \item Structured doubly stochastic clustering (SDS) \citep{wang2016structured}: A DSN method with an enhanced block diagonal structure by incorporating the block diagonal regularization $\|Z \|_{\Boxed{k}}$.
    \item DvD \citep{park2017learning}: A DSN method with an enhanced block diagonal structure based on the Davis-Kahan theorem.
    \item DSNI \citep{julien2022Learning}: A DSN method with an enhanced block diagonal structure by incorporating a idempotent regularization of $Z$.
    \item Doubly stochastic distance clustering (DSDC) \citep{he2023doubly}: A scalable method that replaces the doubly stochastic similarity matrix with a doubly stochastic Euclidean matrix.
\end{itemize}
Other graph-based clustering methods:
\begin{itemize}
    \item SDS \citep{peng2007approximating, kulis2007fast}: A convex relaxation of kernel $k$-means that relaxes the idempotency constraint and the low-rank structure.
    \item DCD \citep{yang2012clustering, yang2016low}: Replaces the $L_2$-norm with KL divergence to measure the discrepancy between the input similarity matrix and the learned similarity matrix.
\end{itemize}

\textbf{Construction of $S$.}
We constructed the similarity matrix $S$ for each dataset using the $q$-nearest neighbors ($q$-NN) graph weighted with the self-tuning method \citep{zelnik2004self}.
Let $x_i \in N_q(x_j)$ represent the sample $x_i$ that belongs to the $q$-NN of $x_j$, then $S$ is defined as:
\begin{equation}
    S_{ij} = \begin{cases}
        \exp\left(-\frac{\|x_i - x_j \|_2^2}{\sigma_i \sigma_j} \right), & \text{if } x_i \in N_q(x_j) \text{ or } x_j \in N_q(x_i) \\
        0, & \text{otherwise}
    \end{cases},
\end{equation}
where $\sigma_i$ was set to the Euclidean distance between $x_i$ and its $7$-th nearest neighbor, and $q$ is chosen as $\lfloor \log_2(n) \rfloor + 1$ as suggested by \citep{von2007tutorial}.

Additionally, in LoRD, we normalized $S \leftarrow S / 1_n^T S 1_n$ because $\forall V \in \Omega, 1_n^T VV^T 1_n = 1$; in KKM and GKKM, we used a fully connected graph (i.e., set $q = n$) because they require $S$ to be positive semidefinite.

\textbf{Initialization method.}
For GKKM, SDS, SC, DirectSC, DSN, SDS, DSNI, and DSDC methods, no random initialization is required, and they only take the constructed $S$ (GKKM, SC, DirectSC, DSN, SDS, and DSNI) or $X$ (DSDC) as input.
For KKM, DCD, NLR and S$^{3}$NMF, we adopted initialization methods provided in the original papers.
For the other methods, we first generated $V \in \R^{n \times k}$ with elements uniformly sampled in the range $[0, 1]$, and
\begin{itemize}
    \item In LoRD and B-LoRD, we used $V^{0} = \mathrm{Sinkhorn}(V)$ (described in Alg. \ref{alg:sinkhorn}) to normalize $V$ onto $\Omega(\mu)$.
    \item In SymNMF and PHALS, we normalized $V^{0} = \frac{\sqrt{\langle S, VV^T \rangle}}{\|VV^T \|_F} V$.
    \item In SR and DBSC, we binarized $V_{ij}^{0} = 1$ if $j = \arg\max_{\hat{j}} V_{i\hat{j}}$, and $0$ otherwise, for each $i$-th row of $V^{0}$.
\end{itemize}

\textbf{Result selection.}
For each method that requires random initialization, we ran $50$ initializations and reported the result corresponding to the minimal or maximal objective function value.
The SDP, DSN, SDS, DvD, and DSNI methods require SC as post-processing to obtain the clustering result, so we ran SC $50$ times and reported the average performance for these methods.

\textbf{Hyperparameters tuning.}
For a fair comparison, we adopted the hyper-parameters tuning methods for DCD, DBSC, DirectSC, S$^{3}$NMF, SDS, DvD, DSNI, and DSDC provided in their original papers.
For KKM, GKKM, SC, SR, SymNMF, PHALS, DSN, and LoRD, there are no hyper-parameters to tune.
For NLR, we carefully tuned the hyper-parameter $\alpha$ and $\beta$ to satisfy the constraint $VV^T 1_n = 1_n$ and guarantee convergence.
For the proposed B-LoRD, we tuned $\tau$ over $\{0.01, 0.02, \dots, 1 \}$ to calculate $\gamma$. 

For simplicity, we present a subset of results here.
We refer readers to the \textit{Appendix} for detailed synthetic experiments in Appendix \ref{sec:add_synthetic}, hyperparameter analysis in Appendix \ref{sec:add_hp}, and convergence analysis in Appendix \ref{sec:add_convergence}.
%, we direct readers to the Appendix.

\begin{comment}
To validate the probabilistic clustering result of LoRD and B-LoRD, we define $P = n V \mathrm{Diag}(\mu)$ and $Z = P P^T$ such that $P_{ij} = \P(y_i = j \;|\; x_i)$ and $Z_{ij} = \P(y_i = y_j \;|\; x_i, x_j)$.
The uncertainty (UC) of $S$ can be estimated as the average relative entropy of $P$ learned by LoRD, i.e.:
\begin{equation}
    \textstyle \mathrm{UC} = -\frac{1}{n \log k} \sum_{i=1}^{n} \sum_{j=1}^{k} P_{ij} \log P_{ij}.
\end{equation}
\end{comment}

\subsection{Clustering Results} \label{sec:clustering}
Table \ref{tab:clustering_ACC} shows the clustering performance in terms of ACC for all methods. 
We also refer readers to Appendix \ref{sec:add_clustering_result} for the results in terms of other metrics, including NMI, PUR, and F1.
\begin{table*}[t]
    \centering
    \tabcolsep=1.3mm
    \caption{Comparisons of clustering performance in terms of ACC. The last column refers to the average ACC of the nine datasets, excluding Isolet, COIL100, and MNIST. The best and second-best results are highlighted in \textbf{bold} and \underline{underlined}, respectively. 
    %The empty results ($-$) indicate that they are not applicable for large datasets due to their high complexity, i.e., $\mathcal{O}(n^3)$.
    }
    \label{tab:clustering_ACC}
    \begin{scriptsize}
    \begin{tabular}{c|cccccccccccc|c}
    \toprule
    % & YaleB & ORL & CHART & USPS-1000 & Isolet & COIL100 & Semeion & MNIST & MNIST-2000 & Wine & Yeast & Ecoli & Average  \\
    & D1 & D2 & D3 & D4 & D5 & D6 & D7 & D8 & D9 & D10 & D11 & D12 & Avg.  \\
    \midrule
    KKM      & $0.485$ &	$0.588$ & $0.835$ & $0.494$ & $0.547$ & $0.473$ & $0.572$ & $0.657$ & $0.606$ & $0.944$ & $0.297$ & $0.569$ & $0.599$ \\
    GKKM     & $0.339$ &	$0.488$ & $0.568$ & $0.507$ & \underline{$0.609$} & $0.238$ & $0.536$ & $-$ & $0.595$ & $0.944$ & $0.317$ & $0.640$ & $0.548$ \\
    SDP      & $0.504$ &	$0.665$ & $0.680$ & $0.539$ & $-$ & $-$ & $0.670$ & $-$ & $0.669$ & $0.916$ & $0.330$ & $0.538$ & $0.603$ \\
    DCD      & $0.461$ &	$0.645$ & $0.570$ & $0.562$ & $0.534$ & $0.562$ & $0.564$ & $0.694$ & $0.680$ & $0.955$ & $0.326$ & $0.574$ & $0.593$ \\
    SC   & $0.466$ &	$0.640$ & $0.568$ & $0.542$ & $0.542$ & \underline{$0.589$} & $0.665$ & $0.682$ & $0.682$ & $0.949$ & $0.333$ & $0.533$ & $0.598$ \\
    SR   & $0.467$ &	$0.638$ & $0.568$ & $0.535$ & $0.530$ & $0.547$ & $0.553$ & $0.670$ & $0.677$ & $0.949$ & $0.371$ & $0.610$ & $0.596$ \\
    NCut   & $0.467$ &	$0.650$ & $0.568$ & $0.521$ & $0.531$ & $0.511$ & $0.554$ & $0.675$ & $0.676$ & $0.944$ & $0.359$ & $0.577$ & $0.591$ \\
    DBSC     & $0.473$ &	\underline{$0.658$} & $0.842$ & $0.552$ & $0.542$ & $0.545$ & $0.618$ & $-$ & $0.625$ & $0.944$ & $0.364$ & $0.539$ & $0.624$ \\
    DirectSC     & $0.429$ &	$0.598$ & $0.645$ & $0.464$ & $0.432$ & $-$ & $0.458$ & $-$ & $0.451$ & $0.899$ & $0.359$ & $0.631$ & $0.548$ \\
    SymNMF   & $0.473$ &	$0.645$ & $0.842$ & $0.549$ & $0.537$ & $0.493$ & $0.615$ & $0.525$ & $0.641$ & $0.916$ & $0.334$ & $0.524$ & $0.615$ \\
    PHALS  & $0.473$ &	$0.645$ & $0.800$ & $0.520$ & $0.546$ & $0.527$ & $0.630$ & $0.610$ & $0.644$ & $0.927$ & $0.362$ & $0.518$ & $0.613$ \\
    S$^{3}$NMF & $0.475$ &	$0.630$ & $0.810$ & \underline{$0.623$} & $-$ & $-$ & $0.704$ & $-$ & $0.664$ & $0.935$ & $0.348$ & $0.578$ & $0.641$ \\
    NLR  & $0.491$ & $0.649$ & $0.645$ & $0.418$ & $0.352$ & $-$ & $0.497$ & $-$ & $0.457$ & $0.938$ & $0.364$ & $0.649$ & $0.568$ \\
    DSN   & $0.469$ & $0.640$ & $0.568$ & $0.547$ & $0.540$ & $0.588$ & $0.664$ & $-$ & $0.681$ & $0.949$ & $0.328$ & $0.534$ & $0.598$ \\
    SDS \ & \underline{$0.503$} & \underline{$0.658$} & $0.847$ & $0.574$ & $-$ & $-$ & $0.693$ & $-$ & $0.666$ & \underline{$0.961$} & \underline{$0.387$} & \underline{$0.711$} & \underline{$0.667$} \\
    DvD  & $0.442$ & $0.601$ & $0.608$ & $0.510$ & $-$ & $-$ & $0.517$ & $-$ & $0.451$ & \pmb{$0.966$} & $0.367$ & $0.616$ & $0.564$ \\
    DSNI  & $0.465$ & $0.632$ & $0.563$ & $0.610$ & $-$ & $-$  & $0.670$ & $-$ & \underline{$0.695$} & $0.899$ & $0.325$ & $0.589$& $0.605$ \\
    DSDC  & $0.405$ & $0.549$ & $0.602$ & $0.478$ & $0.561$ & $0.518$ & $0.619$ & $0.557$ & $0.579$ & $0.888$ & $0.324$ & $0.550$ & $0.555$ \\
    \midrule
    LoRD (ours) & $0.467$ & $0.655$ & \underline{$0.878$} & $0.606$ & $0.593$ & $0.496$ & \underline{$0.755$} & \underline{$0.943$} & $0.657$ & $0.944$ & $0.303$ & $0.455$ & $0.636$ \\
    B-LoRD (ours) & \pmb{$0.515$} & \pmb{$0.685$} & \pmb{$0.905$} & \pmb{$0.740$} & \pmb{$0.644$} & \pmb{$0.647$} & \pmb{$0.783$} & \pmb{$0.964$} & \pmb{$0.745$} & $0.955$ & \pmb{$0.412$} & \pmb{$0.741$} & \pmb{$0.720$} \\
    \bottomrule
    \end{tabular}
    \end{scriptsize}
\end{table*}
From Tbl. \ref{tab:clustering_ACC} (the results of NMI, PUR and F1 in Appendix \ref{sec:add_clustering_result} exhibited similar results), we can observe that
\begin{itemize}
    \item Our B-LoRD significantly outperforms the compared methods, achieving the highest ACC values in most cases ($11/12$). Compare to the second best SDS method, our B-LoRD improve $0.053$ ACC in average.
    \item Our LoRD surpasses the hyperparameter-free methods (KKM, SC, SR, SymNMF, PHALS, and DSC) and remains competitive with block-diagonality-enhanced methods (NLR, DvD, and DSNI). % Because LoRD achieving the highest average ACC, PUR and F1 among these methods.
    \item On balanced datasets, our B-LoRD always achieves the highest ACC values, benefiting from its inherent advantage ($\mu_0 = \mu^{\ast}$) in this case.
    \item On imbalanced datasets, our B-LoRD continues to outperform other methods. For example, on the Yeast and Ecoli dataset—characterized by the highest imbalance ratio (IBR)—B-LoRD achieves the highest ACC, PUR, and F1 values, demonstrating its robustness under data imbalance. Additional results can be found in Appendix \ref{sec:add_hp}.
    \item The block-diagonality-enhanced methods generally outperform others, especially our B-LoRD and SDS, both of which employ $k$-block diagonal regularization (Eq. \eqref{eq:blk-diag}). Compared to SDS, our B-LoRD leverages a low-rank doubly stochastic matrix to simplify computation, thereby improving computational efficiency. 
    %exploits the low-rank doubly stochastic matrix to simplify the computation, thereby improving computational efficiency.
\end{itemize}

\subsection{Correlation between Objective Function Value and ACC}
Fig. \ref{fig:loss_ACC_part} and Table \ref{tab:R_square} show the correlation between the objective function values of models and their ACCs (full results are available in Appendix \ref{sec:add_clustering_result}).

The strength of this correlation was quantitatively measured by the coefficient of determination ($\mathrm{R}^2$).
From these results, we observe that the objective function values of KKM, LoRD, and B-LoRD are highly correlated with the clustering performance, while SR, SymNMF, and PHALS are not.
This may be because SR, SymNMF, and PHALS relax the doubly stochastic constraint, making $V$ unable to represent clusters partition well.
Compared to KKM,  LoRD and B-LoRD further reduce the optimization space by specifying the class prior probability $\mu$, which likely explains why the $\mathrm{R}^2$s of LoRD and B-LoRD are higher than that of KKM.

A directly benfit of the strong correlation is that the final clustering result from multiple initializations can be selected according to its objective function value.

\begin{comment}
\begin{figure}[t]
    \centering
    \includegraphics[width=0.6\linewidth]{graphs/loss_ACC_minipart.eps}
    \caption{The relationship between the objective function value ($x$-axis) and the clustering ACC ($y$-axis) on the MNIST dataset, zoom in for details.}
    \label{fig:loss_ACC_part}
\end{figure}
\end{comment}

\begin{table*}[t]
    \centering
    \tabcolsep=1.4mm
    \caption{The R$^2$ between the objective function value and the ACC learned by $50$ initializations. The values higher than $0.5$ are highlighted by \textbf{bold}.}
    \label{tab:R_square}
    \begin{scriptsize}
    \begin{tabular}{c|cccccccccccc|c}
    \toprule
    $\mathrm{R}^{2}$ & D1 & D2 & D3 & D4 & D5 & D6 & D7 & D8 & D9 & D10 & D11 & D12 & Avg.  \\
    \midrule
    KKM  & $0.08$ & \pmb{$0.53$} & \pmb{$0.75$} & $0.23$ & $0.32$ & \pmb{$0.62$} & $0.24$ & $0.47$ & $0.31$ & \pmb{$0.96$} & $0.12$ & $\!\!-0.02\;\,$ & $0.38 (3)$ \\
    SR & $0.00$ & $0.33$ & $0.35$ & \pmb{$0.94$} & $0.02$ & \pmb{$0.72$} & $\!\!-0.83\;\,$ & NaN & $0.31$ & NaN & $0.01$ & $0.27$ & $\!\!\!0.21\;\;\,$ \\
    SymNMF & $0.37$ & $0.21$ & $0.21$ & $0.18$ & $0.00$ & $0.09$ & $0.30$ & $\!\!-0.02\;\,$ & $0.27$ & \pmb{$0.75$} & $0.00$ & $0.06$ & $\!\!\!0.20\;\;\,$ \\
    PHALS & $0.02$ & $0.28$ & $0.47$ & $\!\!-0.01\;\,$ & $0.07$ & $0.09$ & $0.07$ & $\!\!-0.16\;\,$ & $0.12$ & \pmb{$0.75$} & $0.07$ & $\!\!-0.13\;\,$ & $\!\!\!0.14\;\;\,$ \\
    LoRD (ours) & $0.01$ & $0.02$ & \pmb{$0.71$} & $0.40$ & \pmb{$0.53$} & \pmb{$0.62$} & \pmb{$0.75$} & \pmb{$0.91$} & $0.42$ & \pmb{$1.00$} & $0.24$ & $0.09$ & $0.48 (2)$ \\
    B-LoRD (ours) & $0.14$ & $0.31$ & \pmb{$0.52$} & $0.37$ & \pmb{$0.64$} & \pmb{$0.60$} & \pmb{$0.80$} & \pmb{$0.91$} & \pmb{$0.51$} & \pmb{$1.00$} & $0.45$ & $0.00$ & $0.49 (1)$ \\
    \bottomrule
    \end{tabular}
    \end{scriptsize}
\end{table*}

\begin{figure}[t]
    %\vspace{0pt} % 固定基线
    \centering
    \includegraphics[width=0.8\linewidth]{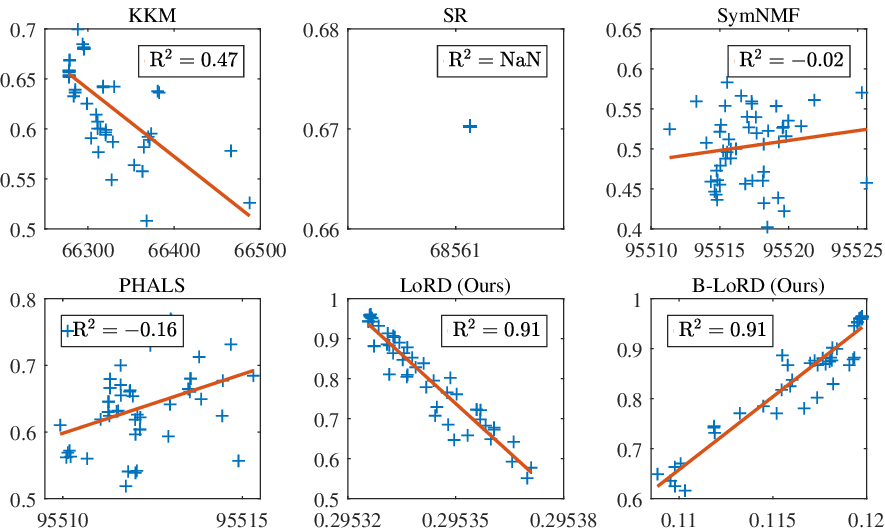}
    \captionof{figure}{The relationship between the objective function value ($x$-axis) and the clustering ACC ($y$-axis) on the MNIST dataset, zoom in for details.}
    \label{fig:loss_ACC_part}
\end{figure}

\begin{comment}
\subsection{How $\tau$ Controls the Block Diagonality of $VV^T$} \label{sec:hp}
\begin{figure}[!t]
    \centering
    \includegraphics[width=0.8\linewidth]{graphs/S_gamma_mini.eps}
    \caption{The visualization of learned $VV^T$ (normalized to $[0, 1]$) with different $\tau$ of B-LoRD on the ORL datasets, zoom in for details.}
    \label{fig:S_gamma}
\end{figure}

To investigate how $\tau$ (used to calculate $\gamma$) controls the block diagonality of $VV^T$ learned by B-LoRD, we visualize $VV^T$ with different $\tau$ in Fig. \ref{fig:S_gamma}, and more results are available in Appendix \ref{sec:add_hp}.
From these results, we can find that:
\begin{itemize}
    \item The block diagonality of $VV^T$ increases as $\tau$ increases, with the rate of increase varying across datasets, which may be related to some properties of $S$.
    \item When $\tau$ is sufficiently small, e.g., $\tau = 0.1$, the learned $V \approx 1_n \mu^T / n$ and $VV^T \approx 1_n 1_n^T / n^2$ with no block diagonality. In theory, this trivial solution always occurs when $\tau = 0$.
    \item When $\tau$ is large, the learned $VV^T$ has high block diagonality, e.g., $\tau = 0.7$ on the ORL dataset. In particular, when $\tau = 1$, the learned $VV^T$ is almost block diagonal, i.e., each row of the learned $V$ has only one non-zero element.
\end{itemize}
More analysis on the impact of $\tau$ can be found in Appendix \ref{sec:add_hp}.
    
\end{comment}

\subsection{Analysis of Hyperparameters} \label{sec:add_hp}
\begin{figure}[t]
    \centering
    \includegraphics[width=\linewidth]{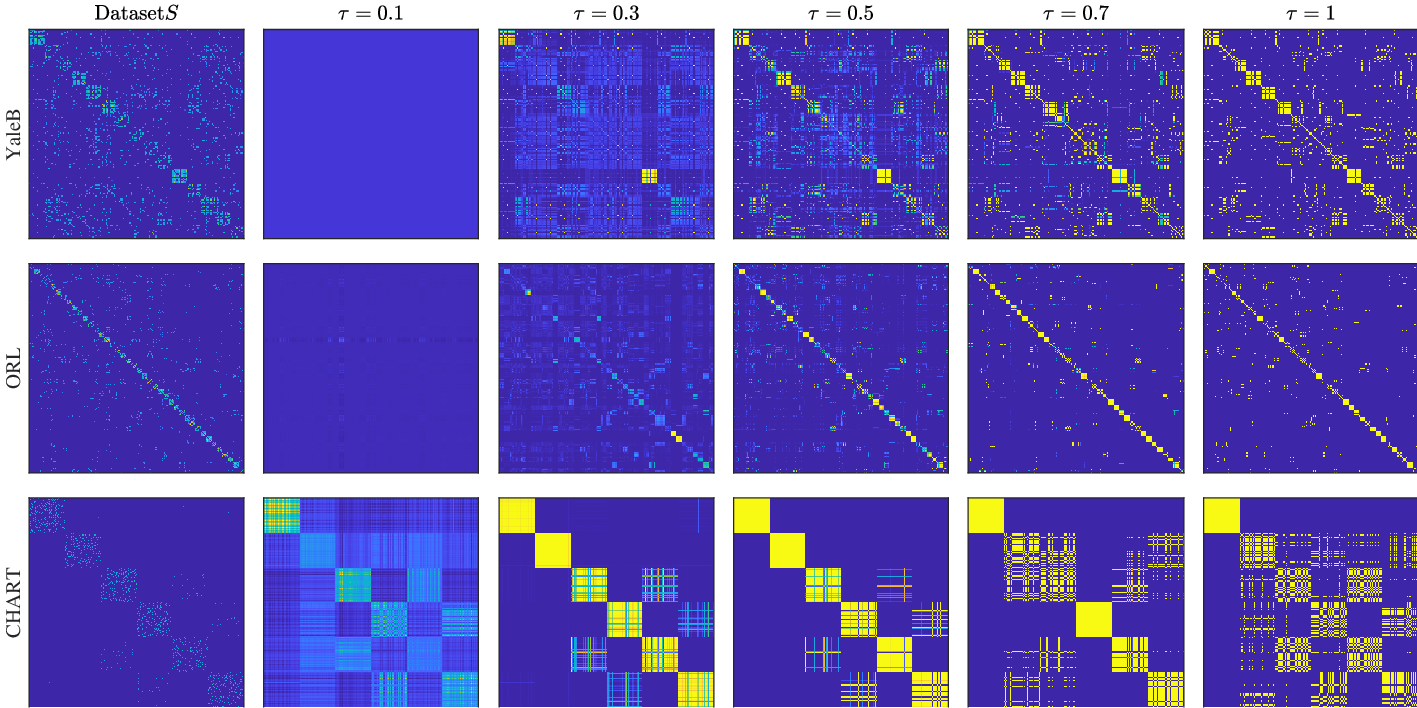}
    \caption{The visualization of learned $VV^T$ (normalized to $[0, 1]$) with different $\tau$ of B-LoRD on the YaleB, ORL, and CHART datasets, zoom in for details.}
    \label{fig:S_gamma}
\end{figure}

\textbf{How does $\tau$ control the block diagonality of $VV^T$?}
To investigate how $\tau$ (used in the computation of $\gamma$) controls the block diagonality of $VV^T$ learned by B-LoRD, we visualize $VV^T$ with different $\tau$ in Fig. \ref{fig:S_gamma}, from which it can be observed that:
\begin{itemize}
    \item The block diagonality of $VV^T$ increases with larger values of $\tau$, although the rate of increase varies across datasets. This variation may be attributed to specific properties of the similarity matrix $S$.
    \item When $\tau$ is sufficiently small (e.g., $\tau = 0.1$), the learned $V \approx 1_n \mu^T / n$ and $VV^T \approx 1_n 1_n^T / n^2$ exhibit no block diagonality. Theoretically, this trivial solution always arises when $\tau = 0$.
    \item When $\tau$ is large, the learned $VV^T$ demonstrates strong block diagonality. For instance, with $\tau = 0.7$ on the ORL dataset, the block structure becomes prominent. Notably, when $\tau = 1$, the learned $VV^T$ is nearly  block diagonal with each row of $V$ containing only a single non-zero element.
\end{itemize}

\textbf{The influence of $\tau$ on clutering accuracy.} 
The result is shown in Fig. \ref{fig:hp}, where $\tau_0 = \frac{\lambda_{\max}(S)}{\lambda_{\max}(S) - \lambda_{\min}(S)}$.
When $\tau > \tau_0$, $\gamma > 0$ and the block diagonality is enhanced; when $\tau < \tau_0$, $\gamma < 0$ and the block diagonality is weakened.
Moreover, we visualized the learned $Z = n^2 V \mathrm{Diag}(\mu \odot \mu) V^T$ of LoRD and B-LoRD on each dataset in Fig. \ref{fig:learned_S}, where the result of B-LoRD corresponds to the best $\tau$ achieving the highest ACC. 
\begin{figure}[h]
    \centering
    \includegraphics[width=\linewidth]{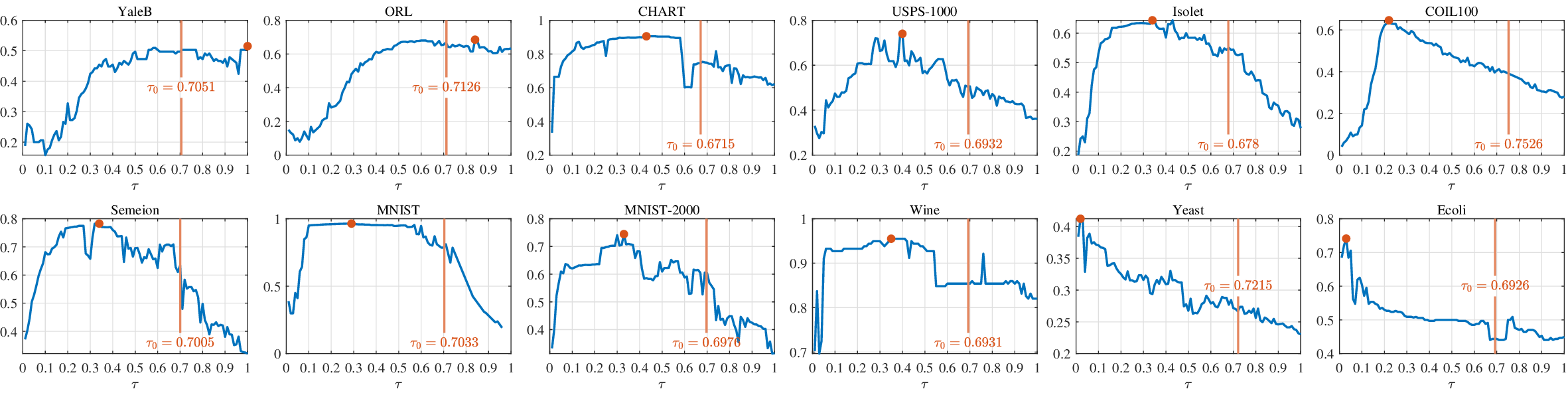}
    \caption{Values of ACC ($y$-axis) of B-LoRD with different values of $\tau$ ($x$-axis).}
    \label{fig:hp}
\end{figure}
\begin{figure}[h]
    \centering
    \includegraphics[width=\linewidth]{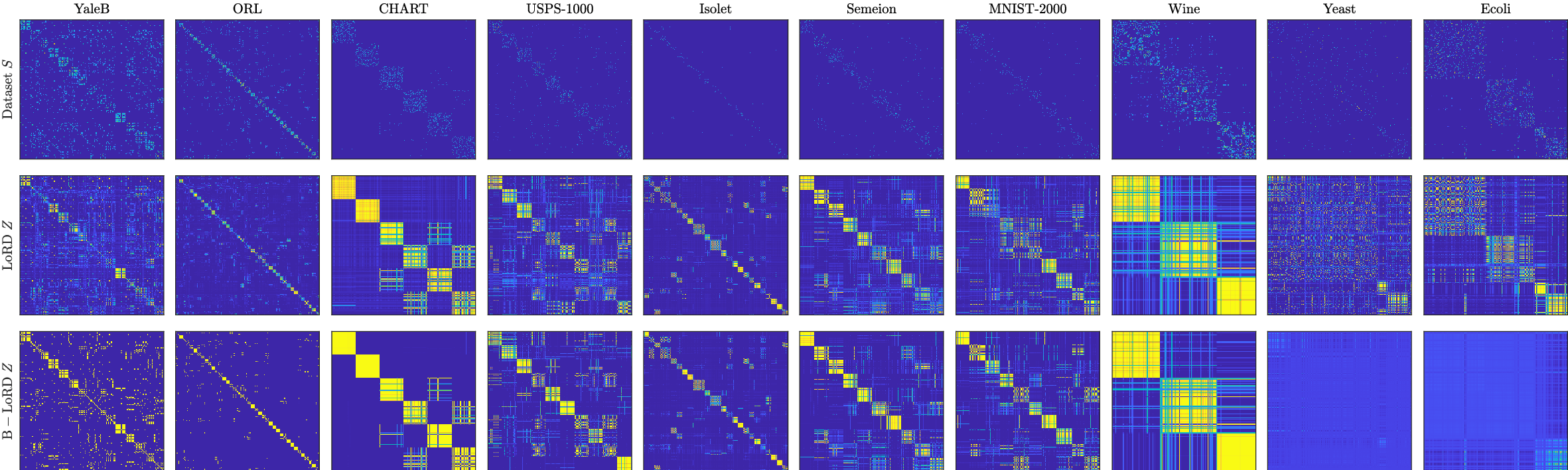}
    \caption{The visualization of learned $Z$ (normalized to $[0, 1]$) of LoRD and B-LoRD on each datasets.}
    \label{fig:learned_S}
\end{figure}
From Fig. \ref{fig:hp} and Fig. \ref{fig:learned_S}, it can be seen that
\begin{itemize}
    \item When the dataset is balanced, the optimal $\tau$ (corresponding to the highest ACC) is generally high, and the learned $Z$ exhibits high block-diagonality, especially on the YaleB, ORL, and CHART datasets. 
    \item When the dataset is imbalanced, B-LoRD cannot find a suitable uniform partition, making small $\tau$ values perform well, especially on the Yeast and Ecoli datasets.
\end{itemize}

\textbf{Practical adaptive selection strategy for $\tau$.}
As shown in Fig. \ref{fig:hp}, the value of ACC is sensitive to the choice of $\tau$, necessitating an adaptive hyperparameter selection strategy. Here we provide two practical guidelines for determining $\tau$:
\begin{itemize}
    \item \textbf{Based on sample size $n$:} Empirical observations suggest that $\tau$ tends to decrease as the sample size $n$ increases. This leads to the first approximation strategy: $\hat{\tau}_1 = \min\{2n^{-0.24}, 1\}$.
    \item \textbf{Based on $n$ and the block-diagonality of $S$:} Additionally, $\tau$ decreases with lower block-diagonality of $S$. 
    We quantify block-diagonality using the metric $b = \sum_{i=1}^{k} \lambda_i(L_S) / \mathrm{Tr}(L_S)$, where $L_S = \mathrm{Diag}(S\mathbf{1}_n) - S$ is the Laplacian of $S$. This metric can be efficiently computed when $S$ is sparse. By combining $b$ and $n$, we propose the second approximation: $\hat{\tau}_2 = \min\{ 0.34 \exp(50 b - 0.03 \log n), 1 \}$.
\end{itemize}
The values of $\hat{\tau}_1$, $\hat{\tau}_2$, and the optimal $\tau^{\ast}$ for each dataset are presented in Table \ref{tab:approx_tau_value}, where MAE stands for Mean Absolute Error. The corresponding clustering ACC values of B-LoRD for $\hat{\tau}_1$, $\hat{\tau}_2$, and $\tau^{\ast}$ are reported in Table \ref{tab:approx_tau_ACC}. From Table \ref{tab:approx_tau_value} and Table \ref{tab:approx_tau_ACC}, it can be observed that when using $\hat{\tau}_1$ and $\hat{\tau}_2$, the MAE values were $0.16$ and $0.13$, respectively, while the average ACC decreases by only $0.05$ and $0.045$ accordingly.
These results validate the effectiveness of the proposed adaptive strategies.
\begin{table*}[htbp]
    \centering
    \tabcolsep=1.6mm
    \caption{The values of $\hat{\tau}_1$, $\hat{\tau}_2$ and $\tau^{\ast}$.}
    \label{tab:approx_tau_value}
    \begin{scriptsize}
    \begin{tabular}{c|cccccccccccc|c}
    \toprule
     & D1 & D2 & D3 & D4 & D5 & D6 & D7 & D8 & D9 & D10 & D11 & D12 & MAE  \\
    \midrule
    $\hat{\tau}_1$  & $0.59$ & $0.47$ & $0.43$ & $0.38$ & $0.23$ & $0.24$ & $0.34$ & $0.14$ & $0.32$ & $0.58$ & $0.35$ & $0.50$ & $0.16$ \\
    $\hat{\tau}_2$  & $0.83$ & $0.95$ & $0.28$ & $0.29$ & $0.26$ & $0.26$ & $0.28$ & $0.24$ & $0.27$ & $0.30$ & $0.28$ & $0.30$ & $0.13$ \\
    $\tau^{\ast}$  & $0.83$ & $0.62$ & $0.44$ & $0.28$ & $0.34$ & $0.22$ & $0.42$ & $0.34$ & $0.38$ & $0.43$ & $0.04$ & $0.03$ & $0$ \\
    \bottomrule
    \end{tabular}
    \end{scriptsize}
\end{table*}

\begin{table*}[htbp]
    \centering
    \tabcolsep=1.4mm
    \caption{The clustering ACC values of B-LoRD corresponding to $\hat{\tau}_1$, $\hat{\tau}_2$ and $\tau^{\ast}$.}
    \label{tab:approx_tau_ACC}
    \begin{scriptsize}
    \begin{tabular}{c|cccccccccccc|c}
    \toprule
    ACC & D1 & D2 & D3 & D4 & D5 & D6 & D7 & D8 & D9 & D10 & D11 & D12 & Avg.  \\
    \midrule
    $\hat{\tau}_1$  & $0.509$ & $0.637$ & $0.905$ & $0.633$ & $0.627$ & $0.629$ & $0.782$ & $0.954$ & $0.716$ & $0.848$ & $0.304$ & $0.500$ & $0.670$ \\
    $\hat{\tau}_2$  & $0.485$ & $0.628$ & $0.890$ & $0.719$ & $0.633$ & $0.608$ & $0.693$ & $0.962$ & $0.702$ & $0.949$ & $0.321$ & $0.509$ & $0.675$ \\
    $\tau^{\ast}$  & $0.515$ & $0.685$ & $0.905$ & $0.740$ & $0.644$ & $0.647$ & $0.783$ & $0.964$ & $0.745$ & $0.955$ & $0.412$ & $0.741$ & $0.720$ \\
    \bottomrule
    \end{tabular}
    \end{scriptsize}
\end{table*}

\subsection{Robustness to Data Imbalance} \label{sec:robust}
In our experiments, we used $\mu_0 = [1 / \sqrt{k}, \dots, 1 / \sqrt{k}]^T$ because $\mu^{\ast} = [\sqrt{\pi_1}, \dots, \sqrt{\pi_k}]^T$ was unknown.
This setting may lead to misleading clustering results when the dataset is significantly imbalanced.
To analyze the performance gap of LoRD and B-LoRD between $\mu_0$ and $\mu^{\ast}$ on imbalanced datasets, we define the imbalance rate (IBR) as: 
\begin{equation}
    \mathrm{IBR} = 1 - \mathcal{H}(\p) / \log k,
    \label{eq:IBR}
\end{equation}
where $\mathcal{H}(\p) = - \sum_{i=1}^{k} \p_i \log \p_i$ is the entropy of $\p$, and the normalization factor $\log k$ ensures that $\mathrm{IBR} \in [0, 1]$.

As shown in Table \ref{tab:ACC_gap}, the performance gap generally increases as IBR increases.
Meanwhile, B-LoRD is more robust to IBR than LoRD, because the block diagonality of $VV^T$ can be adapted by tuning $\tau \in [0, 1]$ to alleviate this effect.
Please see the detailed discussion in Appendix \ref{sec:add_hp}.
\begin{table}[h]
    \centering
    \captionof{table}{The gap of ACC on imbalanced datasets.}
    \label{tab:ACC_gap}
    \begin{scriptsize}
    \begin{tabular}{c|ccccc}
        \toprule
        Datasets & Semeion & MNIST-2000 & Wine & Yeast & Ecoli  \\
        IBR & $3 \times 10^{-5}$ & $0.0014$ & $0.0114$ & $0.2503$ & $0.2704$ \\
        \midrule
        LoRD-$\mu_0$       & $0.755$ & $0.657$ & $0.944$ & $0.303$ & $0.455$ \\
        LoRD-$\mu^{\ast}$  & $0.755$ & $0.669$ & $0.916$ & $0.385$ & $0.717$ \\
        LoRD-gap           & $0$ & $0.008$ & $\!\!-0.028\;\,$ & $0.082$ & $0.262$ \\
        \midrule
        B-LoRD-$\mu_0$     & $0.783$ & $0.745$ & $0.955$ & $0.412$ & $0.741$ \\
        B-LoRD-$\mu^{\ast}$& $0.783$ & $0.742$ & $0.933$ & $0.470$ & $0.801$ \\
        B-LoRD-gap         & $0$ & $\!\!-0.003\;\,$ & $\!\!-0.022\;\,$ & $0.058$ & $0.060$ \\
        \bottomrule
    \end{tabular}
    \end{scriptsize}
\end{table}

\subsection{Convergence Analysis} \label{sec:add_convergence}
The convergence behaviors of the proposed Alg. \ref{alg:main} are shown in Fig. \ref{fig:loss_LoRD} and Fig. \ref{fig:loss_B-LoRD}, with Fig. \ref{fig:loss_B-LoRD} representing the results obtained using the optimal hyperparameter.
From these results, we observe that the objective function value decreases monotonically in Fig. \ref{fig:loss_LoRD} and increases monotonically in Fig. \ref{fig:loss_B-LoRD}, typically reaching convergence within a few hundred iterations.
\begin{figure}[htbp]
    \centering
    \includegraphics[width=\linewidth]{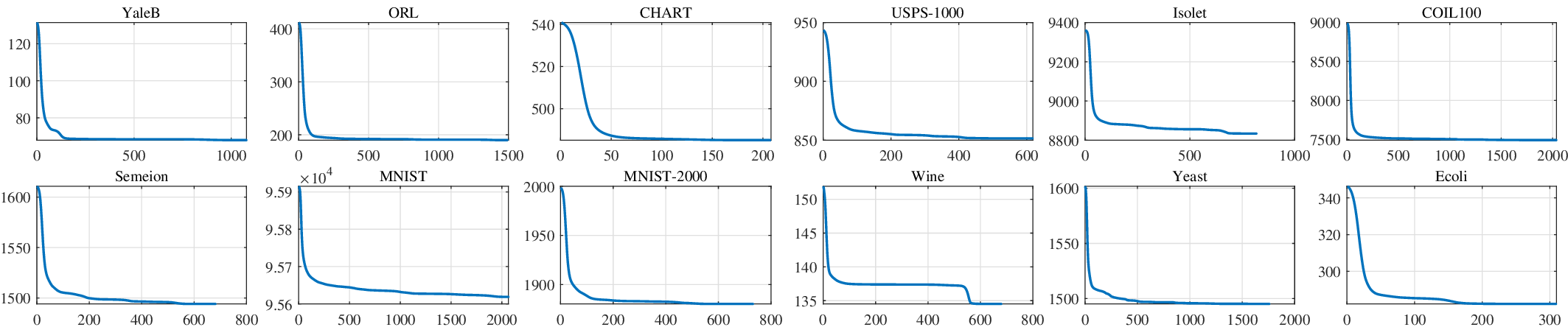}
    \caption{Convergence curves of LoRD on ten datasets. For each dataset, the $x$-axis represents the iteration count, and the $y$-axis represents the objective function of Eq. \eqref{eq:LoRD}.}
    \label{fig:loss_LoRD}
\end{figure}
\begin{figure}[htbp]
    \centering
    \includegraphics[width=\linewidth]{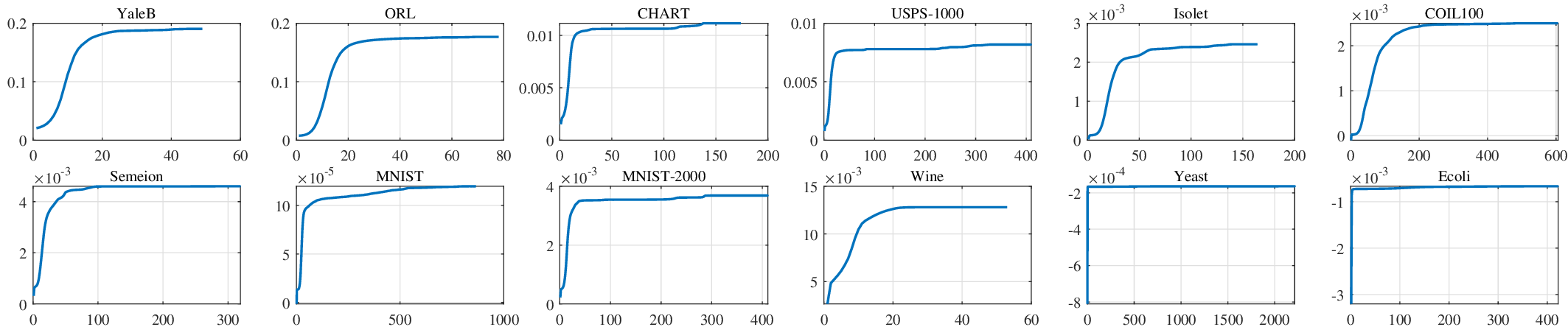}
    \caption{Convergence curves of B-LoRD on ten datasets. For each dataset, the $x$-axis represents the iteration count, and the $y$-axis represents the objective function of Eq. \eqref{eq:B-LoRD}.}
    \label{fig:loss_B-LoRD}
\end{figure}

\subsection{Synthetic Experiment} \label{sec:synthetic}
We generated $200$ samples from four Gaussian distributions: 
$\mathcal{N}\left(\begin{bsmallmatrix} -2 \\ 2 \end{bsmallmatrix}, \begin{bsmallmatrix} 0.25 & 0 \\ 0 & 0.25 \end{bsmallmatrix} \right)$, 
$\mathcal{N}\left(\begin{bsmallmatrix} 2 \\ 2 \end{bsmallmatrix}, \begin{bsmallmatrix} 1 & 0 \\ 0 & 1 \end{bsmallmatrix} \right)$, 
$\mathcal{N}\left(\begin{bsmallmatrix} -2 \\ -2 \end{bsmallmatrix}, \begin{bsmallmatrix} 1 & 0 \\ 0 & 1 \end{bsmallmatrix} \right)$ and
$\mathcal{N}\left(\begin{bsmallmatrix} 2 \\ -2 \end{bsmallmatrix}, \begin{bsmallmatrix} 2.25 & 0 \\ 0 & 2.25 \end{bsmallmatrix} \right)$.
We set five different values of $\p$ to obtain different IBRs, as shown in Table \ref{tab:synthetic_ACC_part}.
\begin{table}[htbp]
    \centering
    \tabcolsep=2.5mm
    \caption{Clustering ACC of synthetic experiment.}
    \label{tab:synthetic_ACC_part}
    \begin{scriptsize}
    \begin{tabular}{cc|cccccc}
    \toprule
    \multirow{2}{*}{$n \p$} & \multirow{2}{*}{IBR} & \multicolumn{2}{c}{GMM} & \multicolumn{2}{c}{LoRD} & \multicolumn{2}{c}{B-LoRD} \\
    & & $\mu^{\ast}$ & $\mu_0$ & $\mu^{\ast}$ & $\mu_0$ & $\mu^{\ast}$ & $\mu_0$ \\
    \midrule
    $[50, 50, 50, 50]$  & $0$      & $0.935$ & $0.935$ & $0.940$ & $0.940$ & \pmb{$0.960$} & \pmb{$0.960$} \\
    $[40, 50, 50, 60]$  & $0.0073$ & \pmb{$0.960$} & $0.945$ & $0.945$ & $0.920$ & $0.945$ & $0.935$ \\
    $[30, 45, 55, 70]$  & $0.0315$ & \pmb{$0.925$} & $0.895$ & $0.860$ & $0.800$ & $0.865$ & $0.880$ \\
    $[20, 40, 60, 80]$  & $0.0768$ & $0.780$ & $0.835$ & $0.870$ & $0.740$ & \pmb{$0.890$} & $0.790$ \\
    $[10, 30, 60, 100]$ & $0.1761$ & $0.840$ & $0.745$ & $0.885$ & $0.620$ & \pmb{$0.900$} & $0.700$ \\
    \bottomrule
    \end{tabular}
    \end{scriptsize}
\end{table}

The clustering results for the case of $n \p = [20, 40, 60, 80]$ are plotted in Fig. \ref{fig:synthetic_res}, where the similarity matrix $S = X^T X$.
\begin{figure}[htbp]
    \centering
    \includegraphics[width=0.8\linewidth]{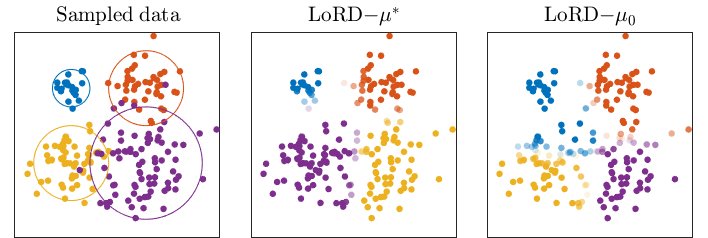}
    \captionof{figure}{The sampled data and the clustering result of LoRD for the case of $n \p = [20, 40, 60, 80]$. 
    The color and opacity of a data point represent its cluster and clustering probability, respectively.}
    \label{fig:synthetic_res}
\end{figure}

From Table \ref{tab:synthetic_ACC_part} and Fig. \ref{fig:synthetic_res}, it can be observed that:
\begin{itemize}
    \item Regardless of whether $\mu^{\ast}$ or $\mu_0$ are provided in LoRD, samples close to the cluster center have high clustering probabilities, while those at the intersection of multiple clusters show low clustering probabilities.
    \item When $\mu_0$ deviates from $\mu^{\ast}$, LoRD fails to find a suitable uniform partition, resulting in a low clustering probability for a large number of samples.
    \item As the IBR increases, the performance gap between LoRD and B-LoRD with given $\mu^{\ast}$ and $\mu_{0}$ becomes more pronounced.
\end{itemize}

Please see the detailed results and analyses in Appendix \ref{sec:add_synthetic}.

\section{Conclusion and Discussion} \label{sec:conclusion}
%In this paper, by relaxing the least important orthonormal constraint of kernel $k$-means, our LoRD and B-LoRD with adjustable block diagonality (Theorem \ref{theorem:blk-diag}) are proposed.

In this paper, we introduced LoRD, a novel graph-based clustering model, by relaxing the least crucial orthonormal constraint of kernel $k$-means, which is further enhanced by integrating adjustable block diagonality, leading to B-LoRD. 
%To ensure numerical solvability, %Theorem \ref{theorem:partition-space} and Theorem \ref{theorem:proba} 
To tackle numerical challenges, we theoretically elucidated how the non-convex doubly stochastic constraint can be reduced to a convex constraint by introducing the class probability parameter $\mu$. Additionally, leveraging the gradient Lipschitz continuity property, we devised a projected gradient descent algorithm for the effective resolution of LoRD and B-LoRD, which theoretically ensures global convergence.
%By showing the property of gradient Lipschitz continuity (Lemma \ref{theorem:lipschitz}), a projected gradient algorithm is proposed to solve LoRD and B-LoRD with globally convergent (Theorem \ref{theorem:convergence}).

%Although our LoRD and B-LoRD are effective,
Despite the effectiveness of LoRD and B-LoRD, a practical hurdle remains as $\mu^{\ast}$ is typically unknown in real-world applications.
Moving forward, our research will delve into methods for accurately estimating $\mu^{\ast}$ to reduce the impact of estimated biases on model performance.
%In the future, we will explore how to estimate $\mu^{\ast}$ and reduce the impact of the estimated bias on performance.
%Moreover, the high correlation between ACC and objective function value motivates us to develop a more efficient algorithm for solving LoRD and B-LoRD.


\newpage

\appendix
\section{Additional Experimental Results}
\subsection{Additional Synthetic Experiment} \label{sec:add_synthetic}
In Tables \ref{tab:synthetic_GMM}, \ref{tab:synthetic_LoRD} and \ref{tab:synthetic_B-LoRD}, we list the clustering performances of GMM, LoRD and B-LoRD in the synthetic experiment, respectively.
Moreover, the clustering results of LoRD are shown in Fig. \ref{fig:synthetic_clusterView_full}.

\begin{comment}
    \begin{table}[htbp]
        \centering
        \caption{Average running time (ms) of synthetic experiment.}
        \label{tab:synthetic_ACC_part}
        \begin{scriptsize}
        \begin{tabular}{c|ccccccc}
        \toprule
        \multirow{2}{*}{$n \p$}  & \multirow{2}{*}{NCut} & \multicolumn{2}{c}{GMM} & \multicolumn{2}{c}{LoRD} & \multicolumn{2}{c}{B-LoRD} \\
        & & $\mu^{\ast}$ & $\mu_0$ & $\mu^{\ast}$ & $\mu_0$ & $\mu^{\ast}$ & $\mu_0$ \\
        \midrule
        $[50, 50, 50, 50]$  & 7  & 16 & 18 & 670 & 673 & 85 & 77 \\
        $[40, 50, 50, 60]$  & 7  & 21 & 17 & 123 & 307 & 77 & 40 \\
        $[30, 45, 55, 70]$  & 8  & 22 & 22 & 83  & 115 & 51 & 104 \\
        $[20, 40, 60, 80]$  & 10 & 29 & 31 & 250 & 49  & 31 & 33 \\
        $[10, 30, 60, 100]$ & 11 & 23 & 52 & 124 & 67  & 34 & 34 \\
        \bottomrule
        \end{tabular}
        \end{scriptsize}
    \end{table}
\end{comment}

\begin{table}[htbp]
    \centering
    \tabcolsep=1.5mm
    \caption{Clustering performance of GMM in synthetic experiment.}
    \label{tab:synthetic_GMM}
    \begin{scriptsize}
    \begin{tabular}{cc|ccc|ccc|ccc|ccc}
    \toprule
    \multirow{2}{*}{$n \p$} & \multirow{2}{*}{IBR} & \multicolumn{3}{c|}{ACC} & \multicolumn{3}{c|}{NMI} & \multicolumn{3}{c|}{PUR} & \multicolumn{3}{c}{F1} \\
    & & $\mu^{\ast}$ & $\mu_0$ & gap & $\mu^{\ast}$ & $\mu_0$ & gap & $\mu^{\ast}$ & $\mu_0$ & gap & $\mu^{\ast}$ & $\mu_0$ & gap \\
    \midrule
    $[50, 50, 50, 50]$  & 0      & 0.935 & 0.935 & 0      & 0.832 & 0.832 & 0      & 0.935 & 0.935 & 0      & 0.879 & 0.879 & 0     \\
    $[40, 50, 50, 60]$  & 0.0073 & 0.960 & 0.945 & 0.015  & 0.877 & 0.847 & 0.030  & 0.960 & 0.945 & 0.015  & 0.916 & 0.887 & 0.029 \\
    $[30, 45, 55, 70]$  & 0.0315 & 0.925 & 0.895 & 0.030  & 0.783 & 0.739 & 0.044  & 0.925 & 0.895 & 0.030  & 0.847 & 0.791 & 0.056 \\
    $[20, 40, 60, 80]$  & 0.0768 & 0.780 & 0.835 & -0.055 & 0.672 & 0.708 & -0.036 & 0.820 & 0.835 & -0.015 & 0.747 & 0.739 & 0.078 \\
    $[10, 30, 60, 100]$ & 0.1761 & 0.840 & 0.745 & 0.095  & 0.657 & 0.506 & 0.151  & 0.840 & 0.745 & 0.095  & 0.713 & 0.623 & 0.090 \\
    \bottomrule
    \end{tabular}
    \end{scriptsize}
\end{table}
\begin{table}[htbp]
    \centering
    \tabcolsep=1.5mm
    \caption{Clustering performance of LoRD in synthetic experiment.}
    \label{tab:synthetic_LoRD}
    \begin{scriptsize}
    \begin{tabular}{cc|ccc|ccc|ccc|ccc}
    \toprule
    \multirow{2}{*}{$n \p$} & \multirow{2}{*}{IBR} & \multicolumn{3}{c|}{ACC} & \multicolumn{3}{c|}{NMI} & \multicolumn{3}{c|}{PUR} & \multicolumn{3}{c}{F1} \\
    & & $\mu^{\ast}$ & $\mu_0$ & gap & $\mu^{\ast}$ & $\mu_0$ & gap & $\mu^{\ast}$ & $\mu_0$ & gap & $\mu^{\ast}$ & $\mu_0$ & gap \\
    \midrule
    $[50, 50, 50, 50]$  & 0      & 0.940 & 0.940 & 0     & 0.838 & 0.838 & 0     & 0.940 & 0.940 & 0     & 0.887 & 0.887 & 0     \\
    $[40, 50, 50, 60]$  & 0.0073 & 0.945 & 0.920 & 0.025 & 0.844 & 0.786 & 0.058 & 0.945 & 0.920 & 0.025 & 0.887 & 0.845 & 0.042 \\
    $[30, 45, 55, 70]$  & 0.0315 & 0.860 & 0.800 & 0.060 & 0.721 & 0.590 & 0.131 & 0.860 & 0.800 & 0.060 & 0.735 & 0.659 & 0.076 \\
    $[20, 40, 60, 80]$  & 0.0768 & 0.870 & 0.740 & 0.130 & 0.724 & 0.552 & 0.172 & 0.870 & 0.740 & 0.130 & 0.749 & 0.597 & 0.152 \\
    $[10, 30, 60, 100]$ & 0.1761 & 0.885 & 0.620 & 0.265 & 0.651 & 0.496 & 0.155 & 0.885 & 0.775 & 0.110 & 0.802 & 0.552 & 0.250 \\
    \bottomrule
    \end{tabular}
    \end{scriptsize}
\end{table}
\begin{table}[htbp]
    \centering
    \tabcolsep=1.5mm
    \caption{Clustering performance of B-LoRD in synthetic experiment.}
    \label{tab:synthetic_B-LoRD}
    \begin{scriptsize}
    \begin{tabular}{cc|ccc|ccc|ccc|ccc}
    \toprule
    \multirow{2}{*}{$n \p$} & \multirow{2}{*}{IBR} & \multicolumn{3}{c|}{ACC} & \multicolumn{3}{c|}{NMI} & \multicolumn{3}{c|}{PUR} & \multicolumn{3}{c}{F1} \\
    & & $\mu^{\ast}$ & $\mu_0$ & gap & $\mu^{\ast}$ & $\mu_0$ & gap & $\mu^{\ast}$ & $\mu_0$ & gap & $\mu^{\ast}$ & $\mu_0$ & gap \\
    \midrule
    $[50, 50, 50, 50]$  & 0      & 0.960 & 0.960 & 0     & 0.884 & 0.884 & 0     & 0.960 & 0.960 & 0     & 0.923 & 0.923 & 0     \\
    $[40, 50, 50, 60]$  & 0.0073 & 0.945 & 0.935 & 0.010 & 0.861 & 0.840 & 0.021 & 0.945 & 0.935 & 0.010 & 0.888 & 0.871 & 0.017 \\
    $[30, 45, 55, 70]$  & 0.0315 & 0.865 & 0.880 & -0.015 & 0.740 & 0.711 & 0.029 & 0.865 & 0.880 & -0.015 & 0.743 & 0.770 & -0.027 \\
    $[20, 40, 60, 80]$  & 0.0768 & 0.890 & 0.790 & 0.100 & 0.752 & 0.631 & 0.121 & 0.890 & 0.820 & 0.070 & 0.787 & 0.682 & 0.105 \\
    $[10, 30, 60, 100]$ & 0.1761 & 0.900 & 0.700 & 0.200 & 0.681 & 0.535 & 0.146 & 0.900 & 0.805 & 0.095 & 0.823 & 0.634 & 0.189 \\
    \bottomrule
    \end{tabular}
    \end{scriptsize}
\end{table}

\begin{figure}[htbp]
    \centering
    \includegraphics[width=1.0\linewidth]{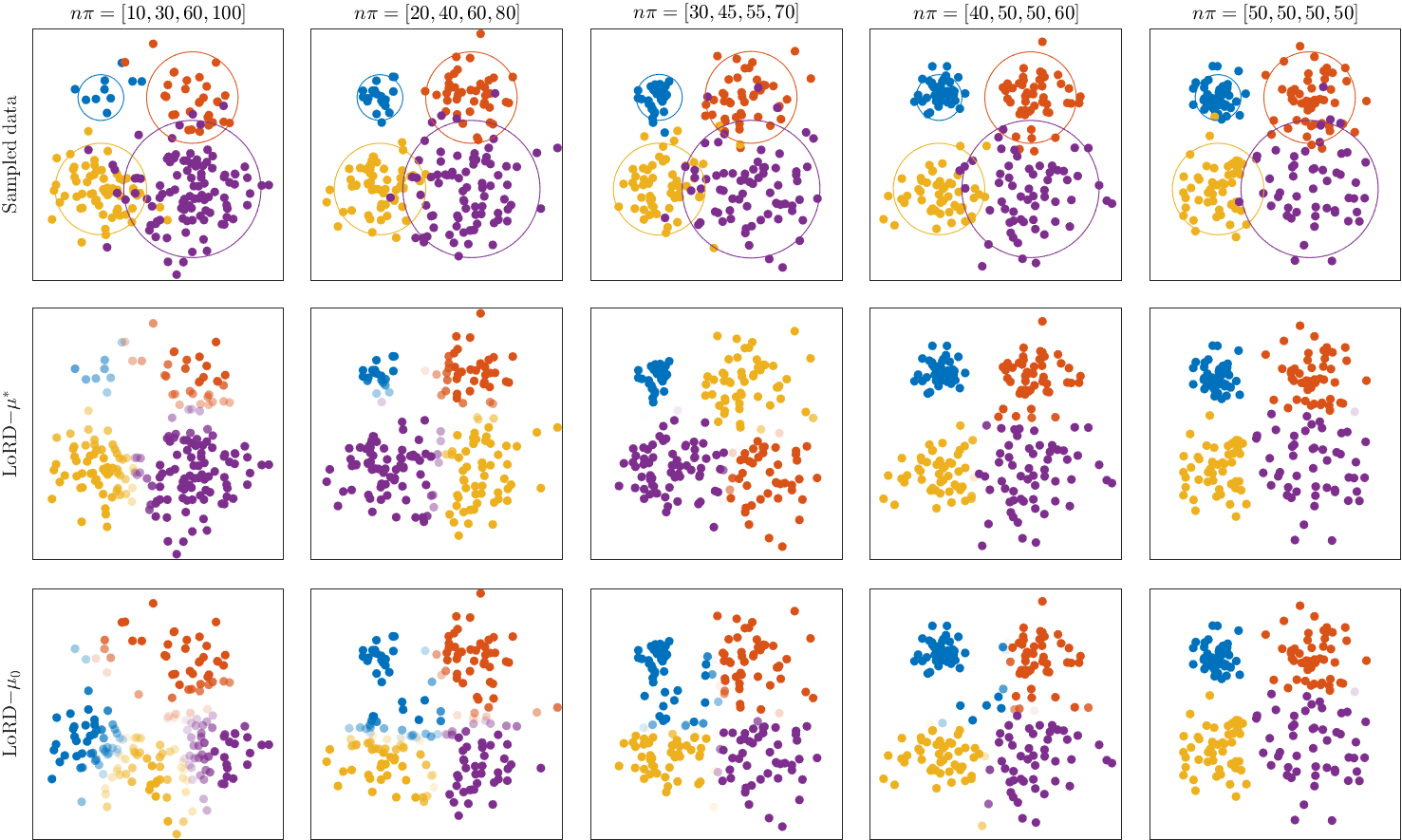}
    \caption{Clustering result of LoRD in synthetic experiment. The four clusters are marked with different colors, and the opacity is set to the cluster probability $\P(y_j | x_i)$.}
    \label{fig:synthetic_clusterView_full}
\end{figure}

\begin{comment}
\begin{figure}[htbp]
    \centering
    \includegraphics[width=0.95\linewidth]{graphs/Similarity.eps}
    \caption{The similarity matrix $S = X^T X$ of data and the learned $Z$ of LoRD in synthetic experiment.}
    \label{fig:synthetic_similarity_full}
\end{figure}
\end{comment}

Additionally, from Tables \ref{tab:synthetic_GMM}, \ref{tab:synthetic_LoRD} and \ref{tab:synthetic_B-LoRD}, it can be seen that B-LoRD is more robust to the deviation between $\mu_0$ and $\mu^{\ast}$.
To study its mechanism, we provide the hyper-parameter analysis of B-LoRD in the synthetic experiment, as shown in Fig. \ref{fig:hp_synthetic}, and the clustering result under the optimal hyper-parameter is shown in Fig. \ref{fig:synthetic_clusterView_B-LoRD}.
From these results, we observe that:
\begin{itemize}
    \item When $\mu^{\ast}$ is known, B-LoRD achieves high ACC when $\tau$ is large, i.e., the learned $VV^T$ exhibits high $k$-block diagonality.
    \item When $\mu_{0}$ deviates from $\mu^{\ast}$, B-LoRD cannot find a suitable uniform partition, resulting in better performance for a smaller $\tau$. This is because, when the $k$-block diagonality of $VV^T$ is weakened, the learned partition ratios do not strictly obey $\mu_0$. For example, as shown in the third row of Fig. \ref{fig:synthetic_clusterView_B-LoRD}, in the cases of $n\pi = [20, 40, 60, 80]$ and $n\pi = [30, 45, 55, 70]$, the partition corresponding to the blue-colored cluster is almost correct. Moreover, for $n\pi = [10, 30, 60, 100]$, the blue-colored cluster vanishes. Therefore, by tuning $\tau$, B-LoRD can enhance or weaken the block diagonality, and thereby reduce the impact of the deviation between $\mu_0$ and $\mu^{\ast}$.
\end{itemize}

\begin{figure}[t]
    \centering
    \includegraphics[width=\linewidth]{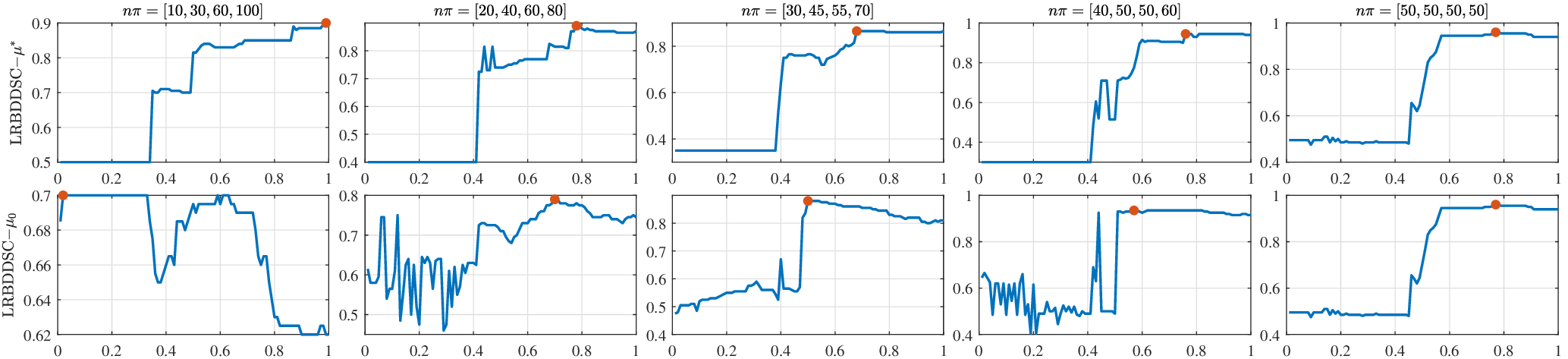}
    \caption{Values of ACC ($y$-axis) of B-LoRD with different values of $\tau$ ($x$-axis) in synthetic experiments. The optimal $\tau$ corresponding to highest ACC is marked with an orange point.}
    \label{fig:hp_synthetic}
\end{figure}

\begin{figure}[htbp]
    \centering
    \includegraphics[width=1.0\linewidth]{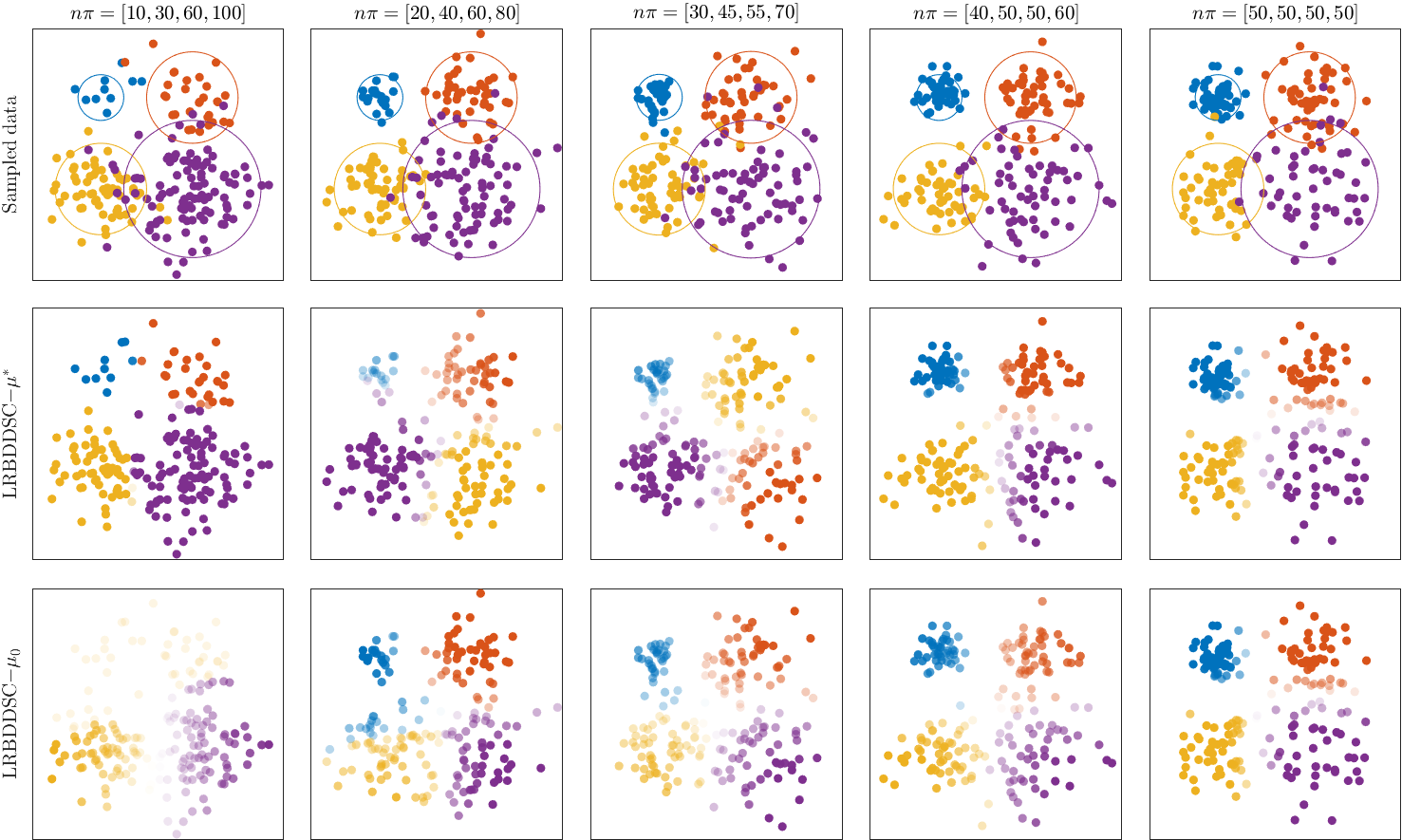}
    \caption{Clustering results of B-LoRD in synthetic datasets. The four clusters are marked with different colors, and the opacity is set to the cluster probability $\P(y_j | x_i)$.}
    \label{fig:synthetic_clusterView_B-LoRD}
\end{figure}

\subsection{Clustering Results} \label{sec:add_clustering_result}
The clustering NMI, PUR and F1 scores of all methods on each dataset are shown in Table \ref{tab:clustering_NMI}, Table \ref{tab:clustering_PUR} and Table \ref{tab:clustering_F1}, respectively.
The analyses of these results are consistent with the discussions in Sec. \ref{sec:clustering}.

\begin{table}[htbp]
    \centering
    \tabcolsep=1.3mm
    \caption{NMI on each dataset.}
    \label{tab:clustering_NMI}
    \begin{scriptsize}
    \begin{tabular}{c|cccccccccccc|c}
    \toprule
    \textbf{NMI} & D1 & D2 & D3 & D4 & D5 & D6 & D7 & D8 & D9 & D10 & D11 & D12 & Avg.  \\
    \midrule
    KKM     & $0.534$ &	$0.765$ & $0.794$ & $0.525$ & $0.741$ & $0.762$ & $0.578$ & $0.582$ & $0.561$ & $0.799$ & $0.239$ & $0.592$ & $0.599$ \\
    GKKM    & $0.415$ &	$0.692$ & $0.788$ & $0.543$ & $0.763$ & $0.622$ & $0.534$ & $-$ & $0.565$ & $0.822$ & $0.239$ & $0.618$ & $0.580$ \\
    SDP    & $0.555$ &	$0.807$ & $0.819$ & $0.576$ & $-$ & $-$ & $0.590$ & $-$ & $0.612$ & $0.763$ & $0.260$ & $0.550$ & $-$ \\
    DCD    & $0.526$ &	$0.798$ & $0.806$ & $0.602$ & $0.758$ & $0.822$ & $0.620$ & $0.744$ & $0.671$ & $0.851$ & $0.256$ & $0.627$ & $0.640$ \\
    SC  & $0.537$ &	$0.798$ & $0.795$ & $0.587$ & $0.764$ & \underline{$0.823$} & $0.660$ & $0.755$ & $0.670$ & $0.825$ & $0.265$ & $0.588$ & $0.636$ \\
    SR  & $0.534$ &	$0.794$ & $0.795$ & $0.579$ & $0.765$ & $0.814$ & $0.607$ & $0.743$ & $0.665$ & $0.825$ & $0.278$ & $0.632$ & $0.634$ \\
    NCut  & $0.532$ & $0.798$ & $0.795$ & $0.569$ & \underline{$0.766$} & $0.793$ & $0.607$ & $0.751$ & $0.663$ & $0.813$ & \pmb{$0.283$} & $0.617$ & $0.631$ \\
    DBSC        & $0.547$ &	$0.801$ & $0.814$ & $0.558$ & $0.758$ & $0.792$ & $0.647$ & $-$ & $0.590$ & $0.828$ & \underline{$0.282$} & $0.581$ & $0.628$ \\
    DirectSC    & $0.518$ &	$0.792$ & $0.809$ & $0.522$ & $0.637$ & $-$ & $0.495$ & $-$ & $0.423$ & $0.718$ & $0.257$ & $0.620$ & $0.573$ \\
    SymNMF  & $0.532$ &	$0.806$ & $0.814$ & $0.539$ & $0.726$ & $0.774$ & $0.614$ & $0.532$ & $0.573$ & $0.781$ & $0.266$ & $0.574$ & $0.611$ \\
    PHALS  & $0.534$ &	$0.803$ & $0.767$ & $0.528$ & $0.755$ & $0.798$ & $0.647$ & $0.689$ & $0.609$ & $0.798$ & $0.279$ & $0.575$ & $0.616$ \\
    S$^{3}$NMF  & $0.525$ &	$0.791$ & $0.816$ & $0.562$ & $-$ & $-$ & $0.665$ & $-$ & $0.636$ & $0.819$ & $0.260$ & $0.592$ & $0.630$ \\
    NLR  & $0.536$ & $0.800$ & $0.591$ & $0.517$ & $0.541$ & $-$ & $0.550$ & $-$ & $0.485$ & $0.802$ & $0.240$ & \pmb{$0.800$} & $0.591$ \\
    DSN  & $0.540$ & $0.799$ & $0.795$ & $0.589$ & $0.765$ & $0.822$ & $0.660$ & $-$ & $0.669$ & $0.825$ & $0.264$ & $0.589$ & $0.637$ \\
    SDS  & \pmb{$0.580$} & \pmb{$0.835$} & \underline{$0.845$} & \underline{$0.651$} & $-$ & $-$ & $0.671$ & $-$ & $0.665$ & \underline{$0.865$} & $0.270$ & \underline{$0.637$} & \underline{$0.669$} \\
    DvD  & $0.499$ & $0.783$ & $0.696$ & $0.478$ & $-$ & $-$ & $0.570$ & $-$ & $0.527$ & \pmb{$0.877$} & $0.232$ & $0.573$ & $0.582$ \\
    DSNI  & $0.540$ & $0.806$ & $0.805$ & $0.647$ & $-$ & $-$ & $0.655$ & $-$ & \pmb{$0.697$} & $0.696$ & $0.263$ & $0.610$ & $0.635$ \\
    DSDC  & $0.487$ & $0.759$ & $0.796$ & $0.453$ & $0.733$ & $0.779$ & $0.560$ & $0.500$ & $0.522$ & $0.713$ & $0.247$ & $0.585$ & $0.569$ \\
    \midrule
    LoRD (ours) & $0.518$ & $0.798$ & $0.827$ & $0.579$ & $0.758$ & $0.774$ & \underline{$0.681$} & \underline{$0.883$} & $0.614$ & $0.807$ & $0.256$ & $0.530$ & $0.623$ \\
    B-LoRD (ours) & \underline{$0.564$} & \underline{$0.824$} & \pmb{$0.850$} & \pmb{$0.678$} & \pmb{$0.794$} & \pmb{$0.831$} & \pmb{$0.696$} & \pmb{$0.910$} & \underline{$0.683$} & $0.853$ & $0.279$ & $0.621$ & \pmb{$0.672$} \\
    \bottomrule
    \end{tabular}
    \end{scriptsize}
\end{table}

\begin{table}[htbp]
    \centering
    \tabcolsep=1.3mm
    \caption{PUR on each dataset.}
    \label{tab:clustering_PUR}
    \begin{scriptsize}
    \begin{tabular}{c|cccccccccccc|c}
    \toprule
    \textbf{PUR} & D1 & D2 & D3 & D4 & D5 & D6 & D7 & D8 & D9 & D10 & D11 & D12 & Avg.  \\
    \midrule
    KKM     & $0.491$ &	$0.628$ & $0.835$ & $0.547$ & $0.607$ & $0.530$ & $0.618$ & $0.657$ & $0.651$ & $0.944$ & $0.513$ & $0.801$ & $0.670$ \\
    GKKM    & $0.364$ &	$0.568$ & $0.667$ & $0.560$ & \underline{$0.643$} & $0.339$ & $0.553$ & $-$ & $0.646$ & $0.944$ & $0.497$ & $0.813$ & $0.624$ \\
    SDP    & \underline{$0.516$} &	$0.687$ & $0.762$ & $0.604$ & $-$ & $-$ & $0.646$ & $-$ & $0.677$ & $0.916$ & $0.526$ & $0.797$ & $-$ \\
    DCD    & $0.461$ &	$0.670$ & $0.667$ & $0.623$ & $0.604$ & $0.641$ & $0.648$ & $0.736$ & $0.728$ & $0.955$ & $0.521$ & $0.836$ & $0.679$ \\
    SC  & $0.467$ &	$0.669$ & $0.667$ & $0.615$ & $0.608$ & \underline{$0.648$} & $0.714$ & $0.750$ & $0.732$ & $0.949$ & $0.519$ & $0.826$ & $0.684$ \\
    SR  & $0.467$ &	$0.670$ & $0.667$ & $0.609$ & $0.601$ & $0.636$ & $0.635$ & $0.737$ & $0.727$ & $0.949$ & $0.514$ & \underline{$0.839$} & $0.675$ \\
    NCut  & $0.467$ & $0.683$ & $0.667$ & $0.593$ & $0.602$ & $0.596$ & $0.638$ & $0.741$ & $0.726$ & $0.949$ & $0.517$ & $0.836$ & $0.675$ \\
    DBSC        & $0.473$ &	$0.685$ & $0.842$ & $0.594$ & $0.612$ & $0.599$ & $0.689$ & $-$ & $0.663$ & $0.944$ & $0.549$ & $0.824$ & $0.696$ \\
    DirectSC    & $0.452$ &	$0.647$ & $0.667$ & $0.497$ & $0.489$ & $-$ & $0.505$ & $-$ & $0.480$ & $0.899$ & $0.520$ & $0.827$ & $0.610$ \\
    SymNMF  & $0.479$ &	$0.690$ & $0.842$ & $0.588$ & $0.587$ & $0.556$ & $0.670$ & $0.586$ & $0.674$ & $0.916$ & $0.530$ & $0.824$ & $0.690$ \\
    PHALS  & $0.479$ &	$0.683$ & $0.800$ & $0.567$ & $0.616$ & $0.588$ & $0.690$ & $0.676$ & $0.690$ & $0.927$ & \underline{$0.553$} & $0.819$ & $0.690$ \\
    S$^{3}$NMF  & $0.485$ &	$0.664$ & $0.817$ & $0.629$ & $-$ & $-$ & $0.712$ & $-$ & $0.678$ & $0.935$ & $0.539$ & $0.827$ & $0.698$ \\
    NLR  & $0.503$ & $0.687$ & $0.645$ & $0.502$ & $0.430$ & $-$ & $0.560$ & $-$ & $0.489$ & $0.938$ & $0.458$ & $0.687$ & $0.608$ \\
    DSN  & $0.470$ & $0.669$ & $0.667$ & $0.619$ & $0.609$ & $0.647$ & $0.713$ & $-$ & $0.732$ & $0.949$ & $0.520$ & $0.826$ & $0.685$ \\
    SDS  & $0.515$ & \pmb{$0.710$} & $0.847$ & $0.638$ & $-$ & $-$ & $0.740$ & $-$ & $0.719$ & \underline{$0.961$} & $0.545$ & \pmb{$0.843$} & \underline{$0.724$} \\
    DvD  & $0.455$ & $0.668$ & $0.680$ & $0.510$ & $-$ & $-$ & $0.602$ & $-$ & $0.522$ & \pmb{$0.966$} & $0.449$ & $0.816$ & $0.630$ \\
    DSNI  & $0.472$ & $0.669$ & $0.667$ & \underline{$0.649$} & $-$ & $-$ & $0.708$ & $-$ & \pmb{$0.747$} & $0.899$ & $0.528$ & $0.817$ & $0.684$ \\
    DSDC  & $0.418$ & $0.600$ & $0.695$ & $0.489$ & $0.609$ & $0.559$ & $0.644$ & $0.601$ & $0.626$ & $0.888$ & $0.521$ & $0.817$ & $0.633$ \\
    \midrule
    LoRD (ours) & $0.479$ & $0.683$ & \underline{$0.878$} & $0.613$ & $0.610$ & $0.538$ & \underline{$0.755$} & \underline{$0.943$} & $0.657$ & $0.944$ & $0.522$ & $0.783$ & $0.702$ \\
    B-LoRD (ours) & \pmb{$0.521$} & \underline{$0.703$} & \pmb{$0.905$} & \pmb{$0.748$} & \pmb{$0.658$} & \pmb{$0.659$} & \pmb{$0.783$} & \pmb{$0.964$} & \pmb{$0.747$} & $0.955$ & \pmb{$0.561$} & $0.833$ & \pmb{$0.751$} \\
    \bottomrule
    \end{tabular}
    \end{scriptsize}
\end{table}

\begin{table}[htbp]
    \centering
    \tabcolsep=1.3mm
    \caption{F1-score on each dataset.}
    \label{tab:clustering_F1}
    \begin{scriptsize}
    \begin{tabular}{c|cccccccccccc|c}
    \toprule
    \textbf{F1} & D1 & D2 & D3 & D4 & D5 & D6 & D7 & D8 & D9 & D10 & D11 & D12 & Avg.  \\
    \midrule
    KKM     & $0.325$ &	$0.453$ & $0.752$ & $0.413$ & $0.518$ & $0.404$ & $0.471$ & $0.540$ & $0.486$ & $0.888$ & $0.243$ & $0.502$ & $0.504$ \\
    GKKM    & $0.165$ &	$0.216$ & $0.688$ & $0.409$ & \underline{$0.566$} & $0.109$ & $0.436$ & $-$ & $0.478$ & $0.889$ & $0.249$ & $0.557$ & $0.454$ \\
    SDP    & $0.339$ &	$0.538$ & $0.718$ & $0.450$ & $-$ & $-$ & $0.478$ & $-$ & $0.513$ & $0.838$ & $0.267$ & $0.445$ & $-$ \\
    DCD    & $0.308$ &	$0.521$ & $0.695$ & $0.496$ & $0.524$ & $0.489$ & $0.497$ & $0.668$ & $0.600$ & $0.909$ & $0.250$ & $0.524$ & $0.533$ \\
    SC  & $0.317$ &	$0.520$ & $0.691$ & $0.481$ & $0.520$ & \underline{$0.528$} & $0.567$ & $0.667$ & $0.600$ & $0.898$ & $0.264$ & $0.476$ & $0.535$ \\
    SR  & $0.317$ &	$0.512$ & $0.691$ & $0.470$ & $0.500$ & $0.446$ & $0.482$ & $0.653$ & $0.592$ & $0.898$ & $0.296$ & $0.541$ & $0.533$ \\
    NCut  & $0.313$ & $0.531$ & $0.691$ & $0.455$ & $0.500$ & $0.313$ & $0.483$ & $0.658$ & $0.591$ & $0.887$ & $0.291$ & $0.523$ & $0.529$ \\
    DBSC        & $0.328$ &	$0.538$ & $0.765$ & $0.447$ & $0.525$ & $0.475$ & $0.562$ & $-$ & $0.528$ & $0.887$ & $0.281$ & $0.467$ & $0.534$ \\
    DirectSC    & $0.293$ &	$0.460$ & $0.699$ & $0.380$ & $0.369$ & $-$ & $0.388$ & $-$ & $0.343$ & $0.811$ & $0.286$ & $0.548$ & $0.468$ \\
    SymNMF  & $0.313$ &	$0.540$ & $0.765$ & $0.436$ & $0.493$ & $0.418$ & $0.536$ & $0.454$ & $0.525$ & $0.838$ & $0.268$ & $0.465$ & $0.521$ \\
    PHALS  & $0.313$ &	$0.534$ & $0.708$ & $0.411$ & $0.523$ & $0.461$ & $0.570$ & $0.605$ & $0.549$ & $0.857$ & $0.281$ & $0.461$ & $0.520$ \\
    S$^{3}$NMF  & $0.320$ &	$0.511$ & $0.757$ & $0.482$ & $-$ & $-$ & $0.605$ & $-$ & $0.583$ & $0.880$ & $0.268$ & $0.509$ & $0.546$ \\
    NLR  & $0.315$ & $0.470$ & $0.514$ & $0.313$ & $0.152$ & $-$ & $0.320$ & $-$ & $0.264$ & $0.877$ & \underline{$0.322$} & $0.470$ & $0.429$ \\
    DSN  & $0.321$ & $0.521$ & $0.691$ & $0.484$ & $0.517$ & $-$ & $0.566$ & $-$ & $0.599$ & $0.898$ & $0.262$ & $0.477$ & $0.535$ \\
    SDS  & \pmb{$0.368$} & \underline{$0.551$} & $0.771$ & $0.521$ & $-$ & $-$ & $0.601$ & $-$ & $0.587$ & \underline{$0.923$} & \pmb{$0.398$} & \underline{$0.673$} & \underline{$0.599$} \\
    DvD  & $0.251$ & $0.302$ & $0.587$ & $0.342$ & $-$ & $-$ & $0.323$ & $-$ & $0.338$ & \pmb{$0.931$} & $0.318$ & $0.533$ & $0.436$ \\
    DSNI  & $0.321$ & $0.514$ & $0.694$ & \underline{$0.523$} & $-$ & $-$ & $0.567$ & $-$ & \underline{$0.622$} & $0.809$ & $0.255$ & $0.541$ & $0.538$ \\
    DSDC  & $0.264$ & $0.414$ & $0.697$ & $0.364$ & $0.516$ & $0.474$ & $0.484$ & $0.451$ & $0.465$ & $0.793$ & $0.243$ & $0.491$ & $0.468$ \\
    \midrule
    LoRD (ours) & $0.308$ & $0.535$ & \underline{$0.802$} & $0.491$ & $0.560$ & $0.436$ & \underline{$0.638$} & \underline{$0.893$} & $0.548$ & $0.888$ & $0.243$ & $0.430$ & $0.543$ \\
    B-LoRD (ours) & \underline{$0.355$} & \pmb{$0.593$} & \pmb{$0.837$} & \pmb{$0.633$} & \pmb{$0.608$} & \pmb{$0.576$} & \pmb{$0.656$} & \pmb{$0.930$} & \pmb{$0.640$} & $0.913$ & $0.365$ & \pmb{$0.744$} & \pmb{$0.637$} \\
    \bottomrule
    \end{tabular}
    \end{scriptsize}
\end{table}

\subsection{Complexity of Dykstra Algorithm \ref{alg:Dykstra}} \label{sec:add_bavg}
As analyzed in Sec. \ref{sec:complexity}, the complexity of the Dykstra Algorithm \ref{alg:Dykstra} is $\mathcal{O}(nk b_{\text{avg}})$, where $b_{\text{avg}}$ is the average iteration count.
In this subsection, we summarize $b_{\text{avg}}$ for LoRD and B-LoRD on each dataset, as shown in Fig. \ref{fig:b_n_k}.
From Fig. \ref{fig:b_n_k}, we observed that:
The $b_{\text{avg}}$ of LoRD is approximately $50$, independent of $n$ but proportional to $k$.
The $b_{\text{avg}}$ of B-LoRD varies more significantly, ranging from approximately $50$ to $500$, and appears to be independent of both $n$ and $k$.
\begin{figure}[htbp]
    \centering
    \subfloat{\includegraphics[width=0.245\linewidth]{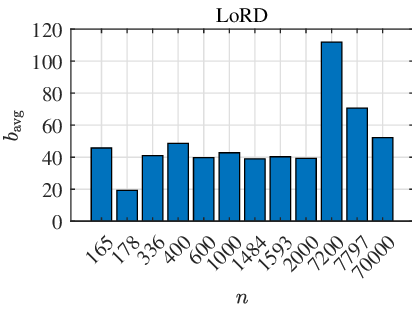}} \hfill
    \subfloat{\includegraphics[width=0.245\linewidth]{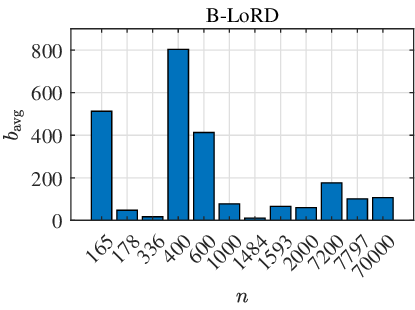}} \hfill
    \subfloat{\includegraphics[width=0.245\linewidth]{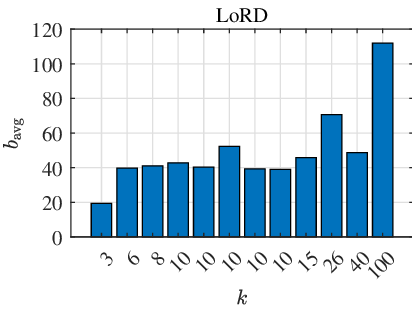}} \hfill
    \subfloat{\includegraphics[width=0.245\linewidth]{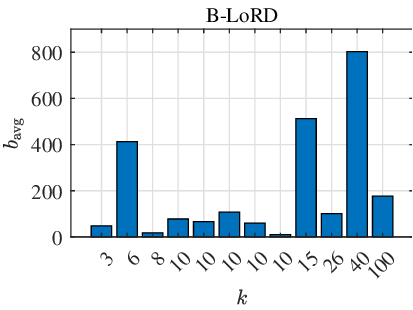}}
    \caption{The correlation between $b_{\text{avg}}$ ($y$-axis) and number of samples $n$ and classes $k$ ($x$-axis, sorted in ascending order).}
    \label{fig:b_n_k}
\end{figure}

\subsection{Correlation between Objective Function Value and ACC}

\begin{comment}
\begin{figure}[htbp]
    \centering
    \subfloat[LoRD]{\includegraphics[width=0.4\linewidth]{graphs/init_ACC_LoRD.eps}} \hspace{5mm}
    \subfloat[B-LoRD]{\includegraphics[width=0.4\linewidth]{graphs/init_ACC_B-LoRD.eps}}
    \caption{Caption}
    \label{fig:init_ACC}
\end{figure}
\end{comment}

The relationship between the objective function value and the clustering ACC of SR, SymNMF, PHALS, LoRD and B-LoRD are described in Fig. \ref{fig:KKM_loss_ACC} to Fig. \ref{fig:B-LoRD_loss_ACC}, respectively.
In general, the correlation (measured by $\mathrm{R}^2$) in KKM, LoRD and B-LoRD is stronger than SR, SymNMF and PHALS, because the doubly stochastic constraint is relaxed in SR, SymNMF and PHALS.
\begin{figure}[htbp]
    \centering
    \includegraphics[width=\linewidth]{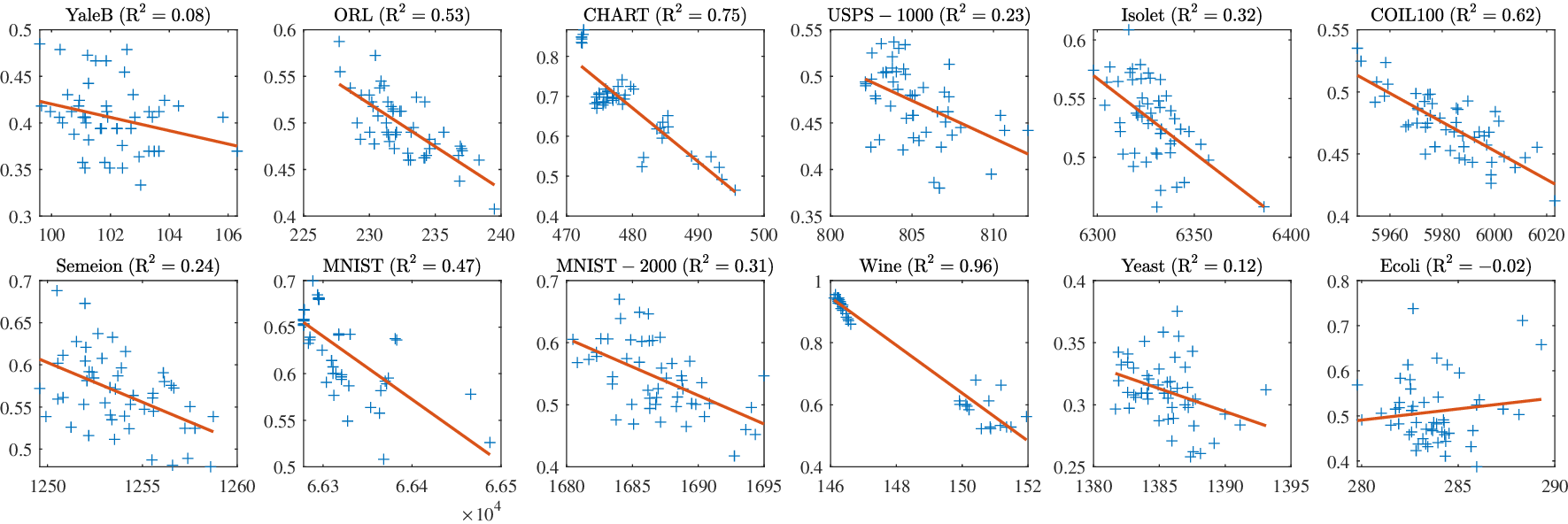}
    \caption{Correlation of objective function and clustering ACC in KKM.}
    \label{fig:KKM_loss_ACC}
\end{figure}
\begin{figure}[htbp]
    \centering
    \includegraphics[width=\linewidth]{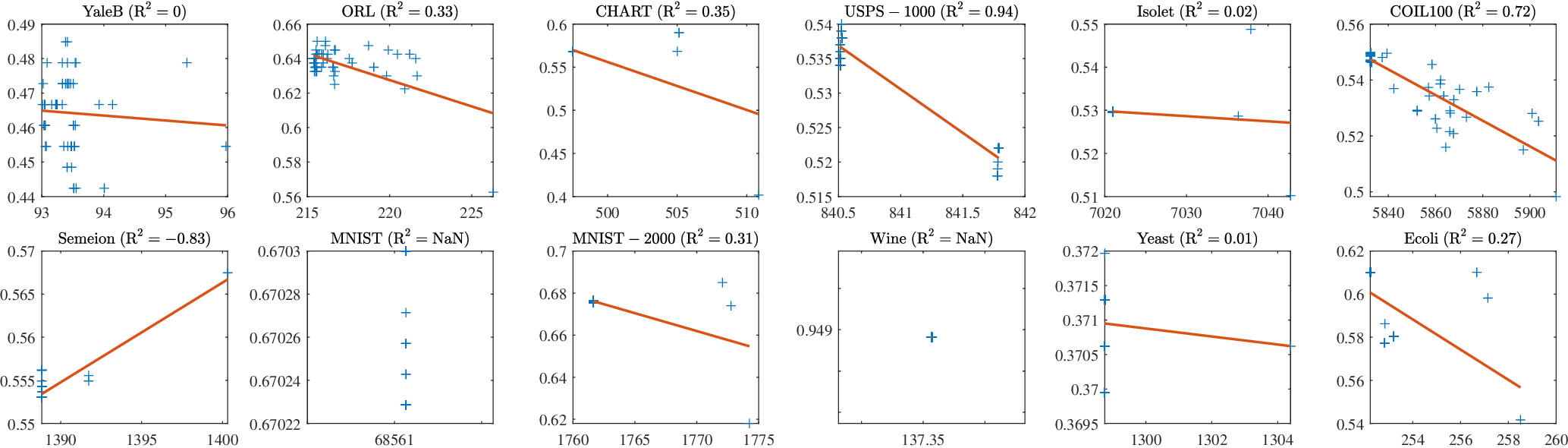}
    \caption{Correlation of objective function and clustering ACC in SR.}
    \label{fig:SR_loss_ACC}
\end{figure}
\begin{figure}[htbp]
    \centering
    \includegraphics[width=\linewidth]{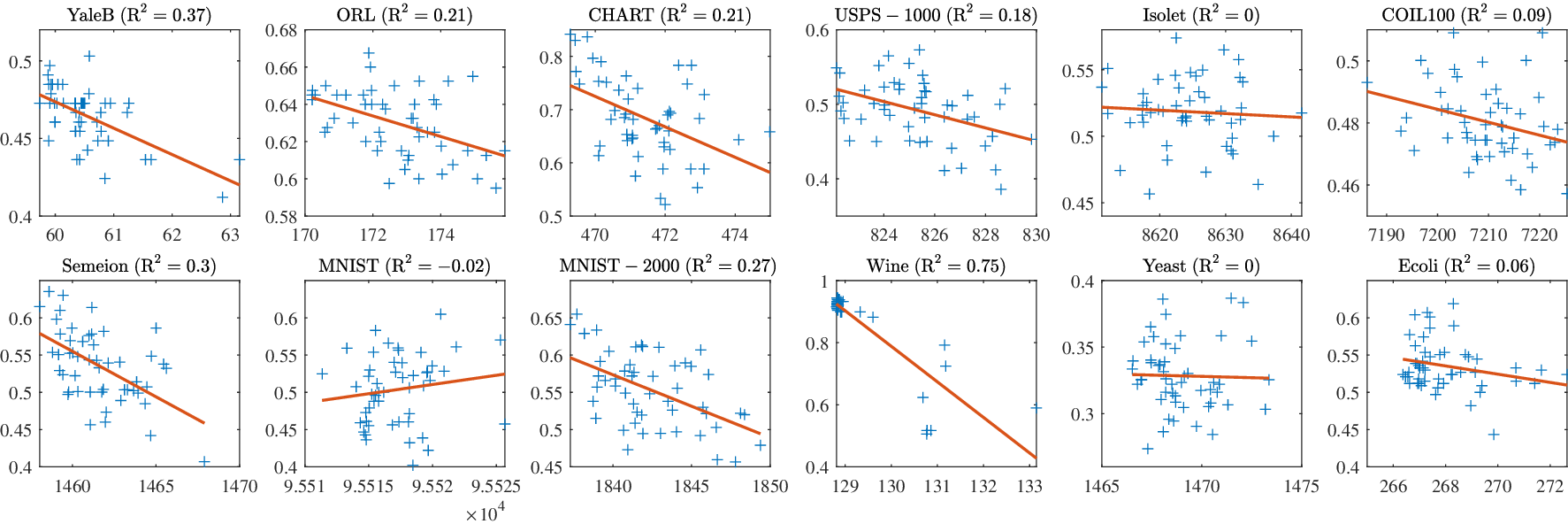}
    \caption{Correlation of objective function and clustering ACC in SymNMF.}
    \label{fig:SymNMF_loss_ACC}
\end{figure}
\begin{figure}[htbp]
    \centering
    \includegraphics[width=\linewidth]{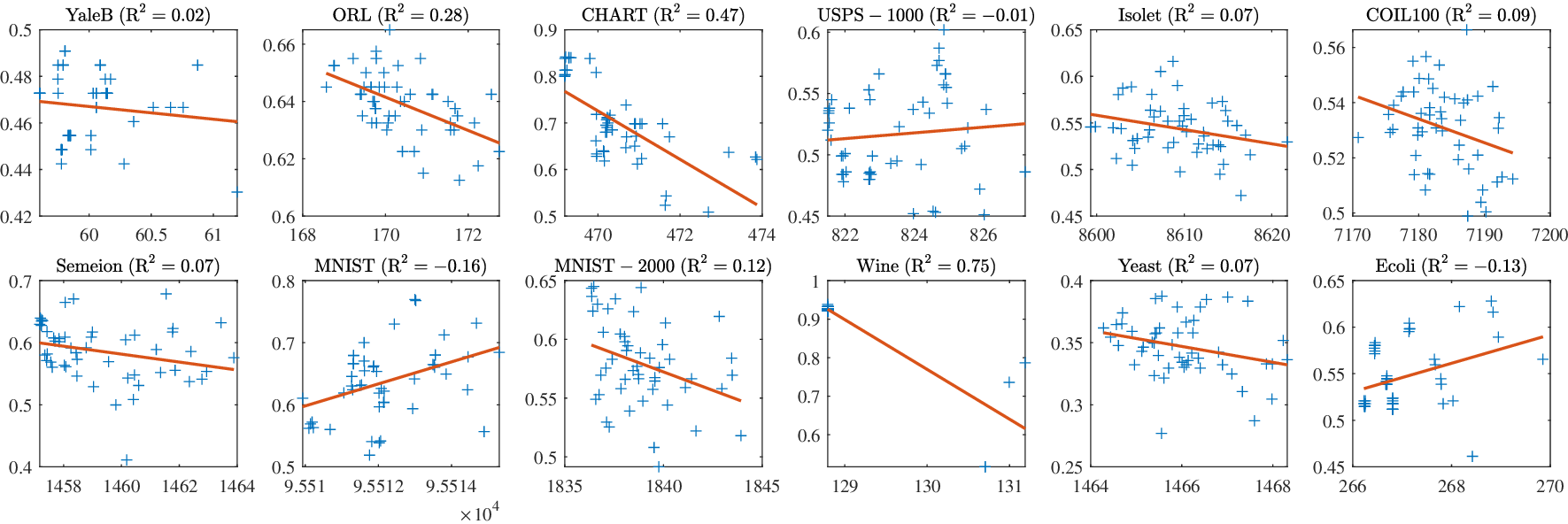}
    \caption{Correlation of objective function and clustering ACC in PHALS.}
    \label{fig:PHALS_loss_ACC}
\end{figure}
\begin{figure}[htbp]
    \centering
    \includegraphics[width=\linewidth]{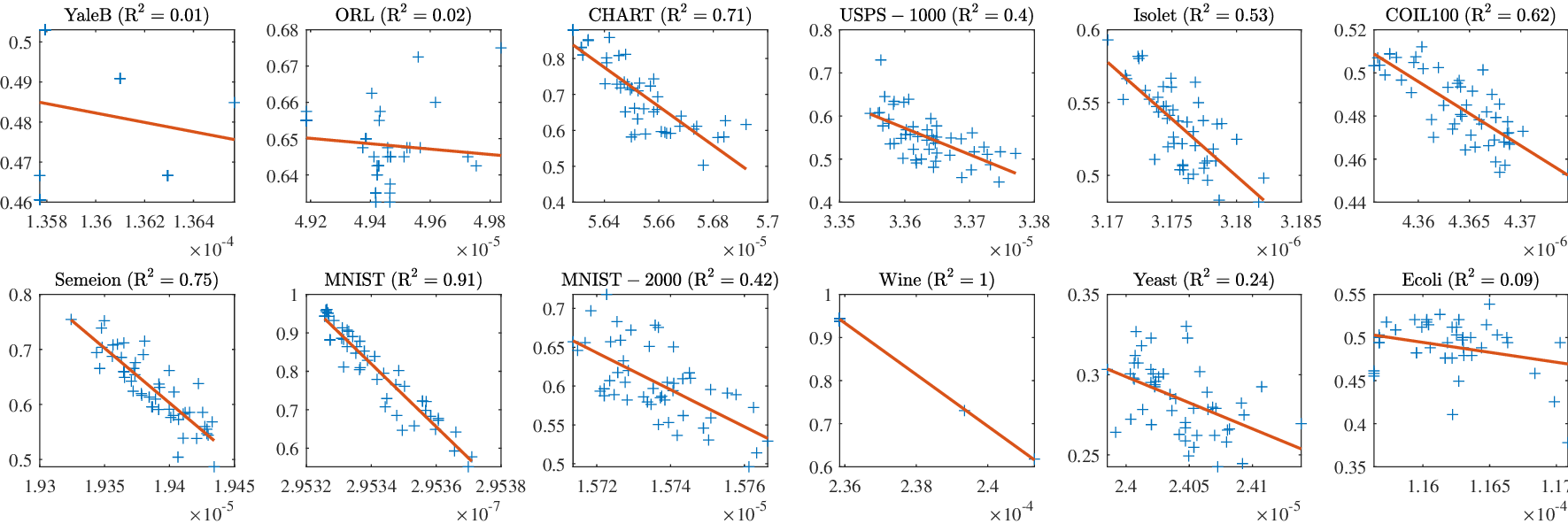}
    \caption{Correlation of objective function and clustering ACC in LoRD.}
    \label{fig:LoRD_loss_ACC}
\end{figure}
\begin{figure}[htbp]
    \centering
    \includegraphics[width=\linewidth]{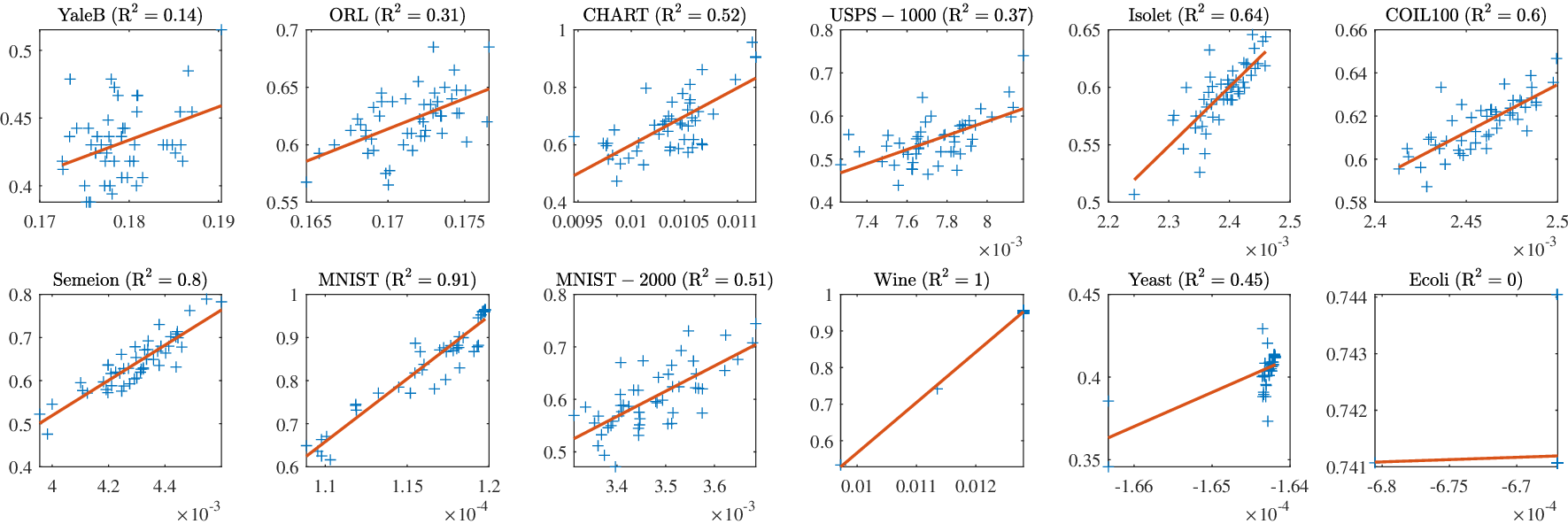}
    \caption{Correlation of objective function and clustering ACC in B-LoRD.}
    \label{fig:B-LoRD_loss_ACC}
\end{figure}

\newpage
\section{Proofs} \label{sec:proofs}
\subsection{Proof of Theorem \ref{theorem:partition-space}} \label{prooftheorem1}
\begin{proof}
    The proof is straightforward; the three conditions in Theorem \ref{theorem:partition-space} are proven as follows: \begin{itemize}
        \item First, for all $\mu \in \mathbb{S}_{+}^{k}$, we can construct $1_n \mu^T / n \in \Omega(\mu)$, which shows that $\Omega(\mu) \ne \emptyset$.
        \item Second, $\Omega(\mu)$ is a subspace of $\Omega$, which shows that $\bigcup_{\mu \in \mathbb{S}_{+}^{k}} \Omega(\mu) \subseteq \Omega$.
        Moreover, for all $V \in \Omega$, we have $1_n^T VV^T 1_n = 1 \Rightarrow V^T 1_n \in \mathbb{S}_{+}^{k}$, thus $V \in \Omega(V^T 1_n)$, which implies that $\Omega \subseteq \bigcup_{\mu \in \mathbb{S}_{+}^{k}} \Omega(\mu)$.
        Thus, $\Omega = \bigcup_{\mu \in \mathbb{S}_{+}^{k}} \Omega(\mu)$ holds.
        \item Third, suppose $V \in \Omega(\mu)$ and $V \in \Omega(\nu)$, we have $V^T 1_n = \mu = \nu$, which contradicts the condition $\mu \ne \nu$.
        Therefore, $\Omega(\mu) \cap \Omega(\nu) = \emptyset$.
    \end{itemize}
\end{proof}

\subsection{Proof of Theorem \ref{theorem:proba}}
\label{prooftheorem2}
\begin{proof}
    Let $\mu = [\sqrt{\P(c_1)}, \dots, \sqrt{\P(c_j)}]^T$, $\P(x_i) = 1_n / n$ and $V_{ij} = \frac{\P(y_i = j | x_i) \P(x_i)}{\sqrt{\P(c_j)}}$.
    According to $\P(y_i = j | x_i) \P(x_i) = \P(x_i | y_i = j) \P(c_j)$, we have:
    \begin{equation}
    \begin{aligned}
        (V \mu)_i &= \P(x_i)\sum_{j=1}^{k} \P(y_i = j | x_i) =\P(x_i) = \frac{1}{n}. \\
        (V^T 1_n)_{j} &= \sum_{i=1}^{n} \frac{\P(y_i = j | x_i) \P(x_i)}{\sqrt{\P(c_j)}}
        = \sqrt{\P(c_j)} \sum_{i=1}^{n} \P(x_i | y_i = j) 
        = \sqrt{\P(c_j)} = \mu_j,
    \end{aligned}
    \end{equation}
    which indicate that $V \in \Omega(\mu)$.
    
    Moreover, let $Z = n^2 V \mathrm{Diag}(\mu \odot \mu) V^T$.
    Under the conditional independence assumption, i.e., $\P(y_i | x_i) = \P(y_i | x_i, x_j)$ and $\P(y_i, y_j | x_i, x_j) = \P(y_i | x_i, x_j) \P(y_j | x_i, x_j)$, we have
    \begin{equation}
        Z_{ij} = \sum_{a=1}^{k} \P(y_i = a | x_i) \P(y_j = a | x_j)
        = \sum_{a=1}^{k} \P(y_i = a, y_j = a | x_i, x_j)
        = \P(y_i = y_j | x_i, x_j).
    \end{equation}
\end{proof}

\subsection{Proof of Theorem \ref{theorem:blk-diag}}
\label{prooftheorem4}
\begin{proof}
    Given $V \in \Omega$, we have the Laplacian of $VV^T$ is:
    \begin{equation}
        L_{VV^T} := \mathrm{Diag}(VV^T 1_n) - VV^T = 1_n / n - VV^T.
    \end{equation}
    Therfore, the first $n - k$ largest eigenvalues of $L_{VV^T}$ are all $1 / n$, and the last $k$ eiganvalues are:
    \begin{equation}
        \lambda_{n - i + 1}(L_{VV^T}) = \frac{1}{n} - \lambda_{i}(VV^T) = \frac{1}{n} - \sigma_i^2(V), \quad i = 1, \dots, k.
    \end{equation}
    Accordingly, $\|VV^T \|_{\Boxed{k}}$ can be simplified as:
    \begin{equation}
    \begin{aligned}
        \sum_{i=1}^{k} \lambda_{n \!-\! i \!+\! 1}(L_{VV^T}) = \frac{k}{n} \!-\! \sum_{i=1}^{k} \sigma_i^{2}(V) \!=\! \frac{k}{n} - \|V \|_F^2.
    \end{aligned}
    \end{equation}
    Moreover, according to $\proj_{\Omega_0(\mu)}(U)$ given in Lemma \ref{lemma:proj}, the least $k$-block diagonal case is
    $1_n \mu^T / n = \proj_{\Omega_0(\mu)}(0_{n \times k}) = \arg\min_{V \in \Omega(\mu)} \|V \|_F^2$,
    where $0_{n \times k}$ is an $n \times k$ matrix with all zeros. 
    For all $\mu \in \mathbb{S}_{+}^{k}$, the $k$-block diagonality of $1_n \mu^T / n$ are equal, i.e., $\|1_n \mu^T / n \|_F^2 = 1 / n$.

    The fully $k$-block diagonal case occurs when $V$ is orthogonal.
    For example, given a partition $G_1, \dots, G_k$, let $V_{ij} = 1 / \sqrt{n \times n_j}$ if $x_i \in G_j$ and zero otherwise.
    Then, $\|V \|_F^2 = \frac{k}{n}$, which implies that $\|VV^T \|_{\Boxed{k}} = 0$.
\end{proof}

\subsection{Proof of Theorem \ref{theorem:lipschitz}}
\label{prooftheorem5}
\begin{proof}
    To show that $\nabla_1 = 4 (VV^T - S)V$ and $\nabla_2 = -2(SV + \gamma V)$ are $L_1$- and $L_2$-Lipschitz continuous on $\Omega$, respectively, we need to prove:
    \begin{equation}
        \forall V, U \in \Omega(\mu), \quad \begin{cases}
            \|\nabla_1(V) - \nabla_1(U) \|_F \leq 4\left(3 / n + \|S\|_{\mathrm{op}} \right) \|V - U \|_F \\
            \|\nabla_2(V) - \nabla_2(U) \|_F \leq 2\left\|S + \gamma I_n \right\|_{\mathrm{op}} \|V - U \|_F
        \end{cases},
    \end{equation}
    where $\|\cdot \|_{\mathrm{op}}$ denotes the operator norm, i.e., the largest singular value of matrix.
    For $\nabla_2$, the proof is straightforward:
    \begin{equation}
        \|\nabla_2(V) - \nabla_2(U) \|_F 
        = \|2(S + \gamma I_n)(V - U) \|_F
        \leq 2\| S + \gamma I_n\|_{\mathrm{op}} \|V - U \|_F.
    \end{equation}
    For $\nabla_1$, we have:
    \begin{equation}
    \begin{aligned}
        \|\nabla_1(V) - \nabla_1(U) \|_F 
        &= \|4 (VV^T V - UU^T U) - 4S(V - U) \|_F \\
        &\leq 4 \|VV^T V - UU^T U \|_F + 4 \|S\|_{\mathrm{op}} \|V - U \|_F.
    \end{aligned}
    \label{aeq:upper_1}
    \end{equation}
    The upper bound of $\|VV^T V - UU^T U \|_F$ can be derived as follows:
    \begin{equation}
        \begin{aligned}
            \|VV^T V - UU^T U\|_F &= \|VV^T(V - U) + V(V - U)^T U + (V - U)UU^T\|_F \\
            &\leq \|VV^T (V - U) \|_F + \|V(V - U)^T U \|_F + \|(V - U)UU^T \|_F \\
            &\leq \left(\|VV^T \|_{\mathrm{op}} + \|V \|_{\mathrm{op}} \|U \|_{\mathrm{op}} + \|UU^T \|_{\mathrm{op}} \right) \|V - U \|_F \\
            &\leq 3 \|VV^T \|_{\mathrm{op}} \|V - U \|_F \\
            &= \frac{3}{n} \|V - U \|_F,
        \end{aligned}
        \label{aeq:upper_VU}
    \end{equation}
    where $\|VV^T \|_{\mathrm{op}} = 1 / n$ because $VV^T 1_n = 1_n/n$.
    Substituting Eq. \eqref{aeq:upper_VU} into Eq. \eqref{aeq:upper_1}, we finally obtain:
    \begin{equation}
        \|\nabla_1(V) - \nabla_1(U) \|_F \leq 4\left(3 / n + \|S \|_{\mathrm{op}} \right) \|V - U \|_F.
    \end{equation}
\end{proof}

\subsection{Proof of Lemma \ref{theorem:convergence}}
\label{prooflemma7}
\begin{proof}
    Suppose $\Omega(\mu) \subseteq \R^{n \times k}$ is colsed, convex and nonempty.
    The projector $\proj_{\Omega(\mu)}(U) = \mathop{\arg\min}_{V \in \Omega(\mu)} \|V - U \|_F^2$ satisfies the following important property:
    \begin{equation}
        \forall V \in \Omega(\mu), U \in \R^{n \times k}, \quad
        \left\langle \proj_{\Omega(\mu)}(U) - U, V - \proj_{\Omega(\mu)}(U) \right\rangle \geq 0.
        \label{aeq:proj_inner}
    \end{equation}
    Given $f(V)$ with a gradient $\nabla(V)$ that is $L$-Lipschitz continuous, $f(V)$ has a quadratic upper bound:
    \begin{equation}
    \begin{aligned}
        f(V^{t+1}) &\leq f(V^{t}) + \langle \nabla(V^{t}), V^{t+1} - V^{t} \rangle + \frac{L}{2} \|V^{t+1} - V^{t} \|_F^2 \\
        &= f(V^{t}) - \frac{L}{2} \|V^{t+1} - V^{t} \|_F^2 + L\langle V^{t+1} - (V^{t} - \nabla(V^{t})) / L, V^{t+1} - V^{t} \rangle
    \end{aligned}
    \end{equation}
    Recall that $V^{t+1} = \proj_{\Omega(\mu)}(V^{t} - \nabla(V^t) / L)$. 
    Substituting $V^{t} - \nabla(V^t) / L$ and $V^{t}$ into $U$ and $V$ in Eq. \eqref{aeq:proj_inner}, we have:
    \begin{equation}
        \langle V^{t+1} - (V^{t} - \nabla(V^{t})) / L, V^{t+1} - V^{t} \rangle \leq 0.
    \end{equation}
    Therefore, we get:
    \begin{equation}
    \begin{aligned}
        f(V^{t}) - f(V^{t+1}) &\geq \frac{L}{2} \|V^{t+1} - V^{t} \|_F^2 \\
        \Longrightarrow \quad f(V^{0}) - f(V^{t}) = \sum_{i=0}^{t} f(V^{i}) - f(V^{i+1}) &\geq  \frac{L}{2} \sum_{i=0}^{t} \|V^{i+1} - V^{i} \|_F^2.
    \end{aligned}
    \end{equation}
    Additionally, by applying $f(V^{0}) - f(V^{\ast}) \geq f(V^{0}) - f(V^{t})$ and $\sum_{i=0}^{t} \|V^{i+1} - V^{i} \|_F^2 \geq (t+1) \min\limits_{0 \leq i \leq t} \|V^{i+1} - V^{i} \|_F^2$, we finally get:
    \begin{equation}
        \min_{0 \leq i \leq t} \|V^{i+1} - V^{i} \|_F \leq \sqrt{\frac{2}{L} \frac{f(V^{0}) - f(V^{\ast})}{t+1}}.
    \end{equation}
\end{proof}

\subsection{Proof of Lemma \ref{lemma:proj}}
\label{prooflemma6}
\begin{proof}
    The projection problem of $U \in \R^{n \times k}$ onto $\Omega_0(\mu)$ is formulated as:
    \begin{equation}
        \proj_{\Omega_0(\mu)}(U) = \mathop{\arg\min}_{V^T 1_n = \mu, V\mu = 1_n/n} \frac{1}{2} \|V - U \|_F^2.
    \end{equation}
    Let $\alpha \in \R^{k}$ and $\beta \in \R^{n}$ be the lagrange multiplier for constraint $V^T 1_n = \mu$ and $V \mu = 1_n /n$, respectively, the Lagrangian $\mathcal{L}(V, \alpha, \beta)$ is:
    \begin{equation}
        \mathcal{L}(V, \alpha, \beta) = \frac{1}{2} \|V - U \|_F^2 + \alpha^T(V^T 1_n - \mu) + \beta^T (V\mu - 1_n / n).
    \end{equation}
    The partial derivative of $\mathcal{L}$ w.r.t. $V$ satisfies:
    \begin{equation}
        \frac{\partial \mathcal{L}}{\partial V} = V - U + 1_n \alpha^T + \beta \mu^T = 0
        \Longrightarrow
        V = U - 1_n \alpha^T - \beta \mu^T.
        \label{aeq:form_V}
    \end{equation}
    By applying the constraint conditions, we have:
    \begin{equation}
        \begin{cases}
            V^T 1_n = \mu \\
            V\mu = 1_n / n 
        \end{cases}
        \Longrightarrow
        \begin{cases}
            U^T 1_n - n\alpha - \mu 1_n^T \beta = \mu \\
            U\mu - 1_n \mu^T \alpha - \beta = 1_n/n
        \end{cases}
        \Longrightarrow
        \begin{bmatrix} n I_k & \mu 1_n^T \\ 1_n \mu^T & I_n \end{bmatrix}
        \begin{bmatrix} \alpha \\ \beta \end{bmatrix}
        =
        \begin{bmatrix} U^T 1_n - \mu \\ U \mu - 1_n / n \end{bmatrix}.
        \label{aeq:lin_alpha_beta}
    \end{equation}
    Therefore, $\alpha$ and $\beta$ can be obtained by solving the linear equation in Eq. \eqref{aeq:lin_alpha_beta}.
    By applying the LDU decomposition of the block matrix, we have:
    \begin{equation}
        \begin{bmatrix} n I_k & \mu 1_n^T \\ 1_n \mu^T & I_n \end{bmatrix} 
        \begin{bmatrix} \alpha \\ \beta \end{bmatrix}
        = 
        \begin{bmatrix} I_k & \mu 1_n^T \\ 0 & I_n \end{bmatrix}
        \begin{bmatrix} n I_k - n \mu \mu^T & 0 \\ 0 & I_n \end{bmatrix}
        \begin{bmatrix} I_k & 0 \\ 1_n \mu^T & I_n \end{bmatrix}
        \begin{bmatrix} \alpha \\ \beta \end{bmatrix}
        =
        \begin{bmatrix} U^T 1_n - \mu \\ U \mu - 1_n / n \end{bmatrix},
    \end{equation}
    which can be further simplified as:
    \begin{equation}
    \begin{aligned}
        \begin{bmatrix} n I_k - n \mu \mu^T & 0 \\ 0 & I_n \end{bmatrix}
        \begin{bmatrix} \alpha \\ 1_n \mu^T \alpha + \beta \end{bmatrix}
        &=
        \begin{bmatrix} I_k & \mu 1_n^T \\ 0 & I_n \end{bmatrix}^{-1}
        \begin{bmatrix} U^T 1_n - \mu \\ U \mu - 1_n / n \end{bmatrix}
        \\
        \Longrightarrow \quad
        \begin{bmatrix} n (I_k - \mu \mu^T)\alpha \\ 1_n \mu^T \alpha + \beta \end{bmatrix} 
        &=
        \begin{bmatrix} (I_k - \mu \mu^T) U^T 1_n \\ U \mu - 1_n / n \end{bmatrix}.
    \end{aligned}
    \end{equation}
    
    Note that $\mathrm{rank}(I_k - \mu \mu^T) = k - 1$, and the null space of $I_k - \mu \mu^T$ is spaned by the vector $\mu$.
    Therefore, the solution of the linear equation Eq. \eqref{aeq:lin_alpha_beta} is:
    \begin{equation}
        \begin{cases}
            \alpha = U^T 1_n / n + c \mu\\
            \beta = (I_n - 1_n 1_n^T / n)U \mu - 1_n / n - c 1_n
        \end{cases},
        \label{aeq:arg_alpha_beta}
    \end{equation}
    where $c \in \R$ is arbitrary.
    Nevertheless, $c$ does not appeared in the expression of $V$.
    To see this, by substituting Eq. \eqref{aeq:arg_alpha_beta} into Eq. \eqref{aeq:form_V}, we finally get:
    \begin{equation}
        V = U + \frac{1_n^T U \mu + 1}{n} 1_n \mu^T - \frac{1_n 1_n^T}{n} U - U \mu \mu^T.
    \end{equation}
\end{proof}

\vskip 0.2in
\bibliography{ref}

\end{document}